\documentclass[twoside,11pt]{article}

\usepackage[hyphens]{url}
\usepackage{jmlr2e}
\usepackage[utf8]{inputenc} 
\usepackage[T1]{fontenc}    
\usepackage{booktabs}       
\usepackage{nicefrac}       
\usepackage{microtype}      
\usepackage{framed}
\usepackage{placeins}  
\usepackage{color}
\usepackage{xcolor}
\usepackage{multirow}
\usepackage{verbatim}
\usepackage{wrapfig,lipsum}
\usepackage{floatrow}
\usepackage{algorithmic}
\usepackage[linesnumbered,ruled,vlined]{algorithm2e}
\usepackage{float}
\usepackage{comment}
\usepackage{subcaption}
\usepackage{amsmath}
\usepackage{booktabs,tabularx}

\newfloatcommand{capbtabbox}{table}[][\FBwidth]
\usepackage{footmisc}  
\usepackage{pifont}
\usepackage{footnote}
\newcommand{\cmark}{\ding{51}}
\newcommand{\xmark}{\ding{55}}

\newcommand{\generategap}[1]{{\textcolor{gray}{[#1]}}}
\newcommand{\markgap}[1]{{\textcolor{red}{[#1]}}}
\newcommand{\markgapsingle}[3]{{\textcolor{red}{[#1]\textsubscript{(#2, #3)}}}}

\usepackage{lastpage}
\jmlrheading{21}{2020}{1-\pageref{LastPage}}{4/20; Revised
10/20}{11/20}{20-326}{Anton Bakhtin, Yuntian Deng, Sam Gross, Myle Ott, Marc'Aurelio Ranzato, Arthur Szlam}
\ShortHeadings{Residual Energy-Based Models for Text}{Bakhtin, Deng,
  Gross, Ott, Ranzato and Szlam}
\firstpageno{1}

\begin{document}

\title{Residual Energy-Based Models for Text}

\author{\name Anton Bakhtin\thanks{\enskip Equal contribution. Corresponding author: Marc'Aurelio Ranzato (\href{mailto:ranzato@fb.com}{\nolinkurl{ranzato@fb.com}}).}$^{*\blacklozenge}$ \enskip Yuntian Deng$^{*\bigstar}$
Sam Gross$^{\blacklozenge}$ \enskip Myle Ott$^\blacklozenge$  \\
{\bf \enskip Marc'Aurelio Ranzato$^\blacklozenge$ \enskip Arthur Szlam$^\blacklozenge$} \\
\email  \texttt{\{yolo,sgross,myleott,ranzato,aszlam\}@fb.com} \enskip \texttt{dengyuntian@seas.harvard.edu} \\
\addr $^{\blacklozenge}$ Facebook AI
Research \enskip \enskip \enskip \enskip \enskip \enskip \enskip \enskip
\enskip  \enskip \enskip \enskip  \enskip  \enskip \enskip \enskip \enskip  \enskip \enskip \enskip 
\enskip  \enskip \enskip \enskip $^\bigstar$Harvard University \\
770 Broadway, New York, NY
10003 \enskip \enskip  \enskip \enskip \enskip \enskip  \enskip \enskip \enskip \enskip  \enskip \enskip \enskip
33 Oxford St., Cambridge, MA 02138}

\editor{Samy Bengio}

\maketitle

\begin{abstract}%
Current large-scale auto-regressive language models ~\citep{radford2019language,liu2018generating,graves2013generating} display impressive fluency and can generate convincing text.  
In this work we start by asking the question: Can the generations of these models be reliably distinguished from real text by statistical discriminators?  
We find experimentally that the answer is affirmative 
when we have access to the training data for the model, and guardedly affirmative even if we do not.  

This suggests that the auto-regressive models can be improved by incorporating the (globally normalized) discriminators into the generative process.  We give a formalism for this using the Energy-Based Model framework, 
and show that it indeed improves the results of the generative models, measured both in terms of perplexity and in terms of human evaluation.
\end{abstract}


\section{Introduction}
Energy-based models (EBMs) have a long history in machine learning~\citep{hopfield, cd, ebm}, especially in the image domain ~\citep{teh03, ranzato13}.  Their appeal stems from
 the minimal assumptions they make about the generative process of the data:  they are a strict generalization of probability models, as the energy function need
not be normalized or even have convergent integral.  Recent works \citep{mordatch19} have demonstrated that they can achieve excellent performance as generative models.
However, despite several promising efforts ~\citep{rosenfeld2001whole, wang2015trans,wang2017learning,wang2017language,wang2018improved}, they still have not been 
 as successful in the text domain as locally-normalized auto-regressive models~\citep{radford2019language}, which generate each word sequentially conditioned on all previous words (such that the probabilities are normalized per word, hence the name ``locally-normalized''). This formulation enables locally-normalized auto-regressive models to be trained efficiently via maximum likelihood and generate samples of remarkable quality.
 

 Nevertheless, in the text domain, local normalization and auto-regression leave room for improvement. 
For example, at training time, standard neural language models (LMs) are conditioned on ground truth context while at test (generation)
time, they are conditioned on their own generations, a discrepancy referred to as exposure bias~\citep{mixer}.
In addition, while heuristics like beam search \citep{sutskever2014sequence} somewhat help re-score at the sequence level,
generation generally lacks long-range coherency because it is produced by the greedy selection of one token at the time
without look-ahead.

The first part of this work  quantifies the space for improvement by investigating to what extent it is possible for a learned discriminator to reliably distinguish real text
from text generated by an auto-regressive model.   We will see in Section \ref{sec:in_domain} that this is indeed possible when the training procedure of the discriminator has access to the corpus used to train the generative model.   This leads immediately to the question of ``how can we build better generative models that close this gap?''; we will study this in Section \ref{sec:ebm4text}.  However, discriminating real vs. machine-generated text is an important task on its own, and has recently gained a lot of attention~\citep{GLTR, radford2019language, zellers19, ippolito2020automatic}.  We thus continue in Section \ref{sec:eval_general}  to make a more in-depth assessment of how robust the discriminators are to changes in the architecture of the generator and the corpus used to train it.  We find, perhaps unsurprisingly, the bigger the discriminator model, the greater the variety of domains in the training data, and  the longer the generated sequence,  the better its performance; but perhaps surprisingly, that the discriminators are remarkably
robust.
 
 In the second part of the work (see Section \ref{sec:ebm4text}) , we interpret the discriminators trained above as EBMs on the residual of the auto-regressive LM, trained using the conditional noise contrastive
estimation objective~\citep{gutmann2010noise}.   Using the LM in this way makes training the EBM easier in two important ways: first, it vastly reduces the size of the space the EBM needs to score.  
Second, it gives a simple method for generating {\em negatives}, i.e., {\it samples} with high scores (low energies).   

 We show how to incorporate the EBM into the probability model of the LM via importance sampling~\citep{impsamp, grover2019bias}, allowing evaluation of the probability of text sequences and (conditional) generation.   Using this formulation, we can estimate perplexity of the residual EBM,
and compare it to other models.  In Section~\ref{sec:exp_generation} we show that our joint model decreases perplexity
on two large datasets, when compared to various auto-regressive language model baselines.
Finally, the EBM generations are significantly preferred by humans
according to our qualitative evaluation. To the best of our knowledge, this is the first time that an EBM has
demonstrated improved generation ability against very strong auto-regressive baselines, both in terms of estimated perplexity and through human evaluation.

\section{Related Work} \label{sec:related}
\textit{General Overview.}\enskip
The key challenge of training EBMs~\citep{hopfield,cd,ebm,ebm_unsup} is mining for good negatives. This can be accomplished explicitly by fantasizing inputs where
the energy should be increased or implicitly via global constraints such as sparsity~\citep{ebm_unsup}.
Methods attempting at directly maximizing the likelihood of the data using gradient descent usually require sampling from the distribution induced by the model \citep{hinton2002training,carreira2005contrastive,mordatch19,xie2019learning, wang2017learning}.
Unfortunately, gradient-based MCMC approaches like Hybrid Monte Carlo~\citep{teh03} and Langevyn dynamics~\citep{ebm_unsup,
mordatch19,xie2016theory,xie2017synthesizing, xie2019learning, xie2018learning, gao2018learning,nijkamp2019learning} are not applicable when the input is discrete like in text applications.
Other approaches like Gibbs sampling~\citep{cd} were applied to binary inputs but do not scale well to large
dictionaries once the energy function is a large
bidirectional Transformer model~\citep{vaswani2017attention} like the one used in this work.

Since our EBM is learned after the generator has been trained, it learns from the {\em residual error}
of the generator, and therefore, our training procedure is a particular instance of a ``cascade'' model~\citep{cascade}
and ``boosting''~\citep{boosting}.

Generative Adversarial Networks~\citep{gan,yu2017seqgan,de2019training,scialom2020coldgans,caccia2018language} also relate to EBMs, except that in GANs the discriminator is only used at training time, and the goal is to use it to help improve the generator. In our work, the pretrained locally normalized language model can be seen as
as fixed generator and our goal is to use the discriminator to improve generations from it.

\citet{azadi2018discriminator} also share our goal but their generator is not locally normalized and they
propose to improve the sampling from the generator by using the discriminator for rejection sampling. Similar to our work,
\citet{grover2019bias} propose to use the discriminator to de-bias a pretrained generator using importance sampling.
We adapt this work to the application of text generation. In particular, we adopt the conditional noise contrastive estimation
(NCE) objective~\citep{ma18, gutmann2010noise}
to our residual model energy function and then sample from the joint model using importance sampling.

\noindent
\textit{Generalization of Real/Fake Discrimination.}\enskip
An important contribution of this work is an empirical study of the generalization ability
of EBMs applied to text modeling. Given our residual formulation, this reduces to analyzing the generalization
ability of the model to discriminate real versus machine-generated text.
Several recent works have also studied whether machine generations can be detected automatically,
but they do not study how these findings {\it generalize} to settings where generator architectures
and corpora are different between training and test time. For example, \citet{zellers19} (GROVER) assume that the
generator is known and apply fine-tuning to train a discriminator initialized with the generator.
Similarly, \citet{GLTR} (GLTR)  assume knowledge of the generator; these authors say
``{\em We further hypothesize that these methods generalize to black-box scenarios,
as long as the fake text follows a similar sampling assumption and is generated by a large language model}'';
our work answers precisely this question, providing a rigorous experimental protocol and quantitative results. Concurrent to our work, \citet{ippolito2020automatic} study the generalization ability of discriminators on generations with different decoding schemes, complementing the results described here.

Finally, there has been a release of a dataset of the GPT-2 language model generations~\citep{gpt2gen}
for the purpose of training discriminators capable of detecting machine-generated text. While we share
the same motivation, our work is a much broader investigation on the topic. We assess generalization
of several discriminator architectures to not just one but several kinds of generators and
corpora used during training (including GPT-2).

\noindent
\textit{EBMs for Modeling Text.}\enskip
Several variants of auto-encoders have been investigated
for representing and generating text~\citep{textvae,aae,he2019lagging}, but most have not shown significant improvements
in terms of perplexity, with the notable exception of \cite{he2019lagging} which outperforms an LSTM-based language model. However, most works have so far been applied to relatively small datasets and compared with small-to-medium-sized baseline models only~\citep{textvae,aae,he2019lagging}.

 We adopt the conditional noise contrastive estimation (NCE) objective \citep{ma18,gutmann2010noise} to our residual model energy function and then sample from the joint model using importance sampling. While \citet{ma18} used conditional NCE to predict the next word in a sequence,
we apply it to produce a whole sequence at once with
the pretrained auto-regressive language model as the noise distribution.

Our approach appears similar to discriminative reranking approaches used in the parsing and machine translation
community~\citep{shen2004discriminative, naskar2020energy}. However, our approach provides a generative model, and parameters/hyper-parameters
are directly tuned to close the gap between the model distribution and the data distribution,
rather than relying on surrogate ranking losses. This approach is also related to other sequence-level training
objectives~\citep{seqtr18}, with the major difference that those works aim at improving the baseline model,
but generation at test time is still only auto-regressive.


EBMs have been used for sequence modeling~\citep{rosenfeld2001whole,wang2015trans,wang2017learning,wang2017language,wang2018improved,parshakova2019global}. In particular, our residual modeling form and the training algorithm is the same as in \citet{wang2018learning}, where they used an LSTM as the generator and a CNN-LSTM as the energy function, and showed significant gains compared to LSTM baselines in speech recognition.  Our work builds on these prior works and develops new lower and upper bounds for the log-probability under the joint model, which makes it possible to show that the residual EBM approach gets better perplexity. We also develop an importance weighting sampling scheme used at generation time, which is focused on conditional generation as opposed to rescoring in speech recognition \citep{wang2018learning}. The residual EBM formalism makes it very natural to use BERT~\citep{bert, liu2019roberta} for language modeling, and we show that empirically this type of approach can outperform modern state-of-the-art language modeling baselines, both in terms of perplexity, and through human evaluation.

\section{Basic Setup} \label{sec:basicsetup}
We are interested in modeling discrete sequences $x_{p+1}, \ldots , x_T$, conditioned on a (perhaps empty) prefix $x_1,\ldots , x_p$, with $x_j \in V$, where $V$ is the vocabulary.   Each $x$ will be a byte pair encoded (BPE) token~\citep{sennrich2015neural}.   

A standard approach to this problem is to use a locally-normalized auto-regressive model $P = P_{\phi}(x_{i+1} | x_0, \ldots , x_i)$ that produces the probability of each token in the sequence conditioned on the previously seen tokens.  Here  ``locally-normalized'' refers to the model outputting a probability distribution for each token; and ``auto-regressive'' refers to the model conditioning on the tokens earlier in the sequence to produce this distribution.  The parameters $\phi$ of the model, which in this work will be parameterized as a neural network, are fitted via maximum-likelihood training on a corpus of sequences. For a given sequence the loss can be written as: 
\begin{equation}
\mathcal{L}(\phi) = - \log P_{\phi}(x_{p+1}, \ldots , x_T | x_1,\ldots , x_p) = \sum_{i=p+1}^T - \log P_{\phi}(x_i | x_1,\ldots , x_{i-1})
\end{equation}

Large locally normalized auto-regressive language models trained on vast amounts of data have recently been shown to generate fluent and coherent text~\citep{radford2019language, zellers2019defending, ctrl}.
A natural question to ask is whether such generations can be {\em automatically} detected, as this could be leveraged to further improve such generative models. 
In other words, if a classifier can detect machine-generated text, we can then use the score of such classifier to improve text generation. 

In the next sections, we show that generations from even large models can be discriminated from real text by such classifiers. This will motivate the use of such classifier scores to improve text generation, as described in 
Section~\ref{sec:ebm4text}.  

\section{Can Machine Learning Recognize Generated Text?} \label{sec:realfake_setting}
In this section we will study the ability of classifiers to discriminate between real text and text generated by a model.   
We will use 
\begin{equation}
   \mathbb{E}_{x_+ \sim P_{data}} \log \frac{1}{1+\exp(- s_\theta(x_+))} +  \mathbb{E}_{x_- \sim P_{\phi}} \log \frac{1}{1+\exp(s_\theta(x_-))}
\label{eq:loss_realfake}
\end{equation}
as the objective function for training the discriminators.  Here, 
$x_+$ is a positive sequence taken from the human-generated training set $P_{data}$ (see Section~\ref{sec:data}), 
$x_-$ is a negative sequence drawn
from the locally-normalized auto-regressive pretrained language model $P_{\phi}$ for a given ground truth prefix 
(see Section~\ref{sec:generarch}), and $s_\theta$ is the 
un-normalized score of the classifier, a neural network whose architecture is described in Section~\ref{sec:expscorer}. This objective function is the binary cross entropy loss assuming real data labeled as 1 and generated as 0.

The goal of learning is to find the parameters $\theta$ of the classifier that generalize well at test time.  
The generalization ability of $s_\theta$ to a variety of positive and negative sample distributions points to fundamental 
deficiencies of the generator $P_{\phi}$. Such deficiencies are going to be exploited to improve generation 
as later discussed in Section~\ref{sec:ebm4text}. 
In this section, we are going to consider negatives produced by different  generator training runs, generator architectures, and generator training corpora. We also vary our sampling strategies: we define 
\begin{equation*}
    P_\phi^{(k)}(x_t | x_{<t}) \propto \begin{cases}{P_\phi(x_t | x_{<t})} &  x_t \in k\arg\max_{x_t'} P_\phi(x_t' | x_{<t})\\0 & \text{o.t.}\end{cases},
\end{equation*}
and either sample from  $P_\phi^{(k)}(x_t | x_{<t})$ with $k=10$ which we refer to as ``top-k sampling'' \citep{fan2018hierarchical}, or sample from the original distribution $P_\phi(x_t | x_{<t})$ which is the default setting or is sometimes referred to as ``sampling with temperature 1'' for distinction. 

\subsection{Corpora} \label{sec:data}
\begin{table}[t]
\center
\begin{tabular}{lrrr}
\toprule
\bf Dataset & Train & Valid & Test \\
\midrule
Books &  690 & 7.3 & 8.0 \\
CCNews & 21718 & 1.0 & 3.4 \\
Wikitext &   113 & 0.2& 0.3\\
\bottomrule
\end{tabular}
\caption{\small Number of BPE tokens in millions for each dataset.}
\label{tab:datasets}
\end{table}

We train models on three corpora coming from different domains:
\begin{itemize}
\item {\bf Books:} The Toronto books corpus described in \citet{Zhu_2015_ICCV, kiros2015skip}, which consists of fiction books in 16 different genres,
totaling about half a billion words.
\item {\bf CCNews:} We collect a de-duplicated subset of the English portion of the CommonCrawl news dataset \citep{ccnews}, which totals around 16 Billion words.
\item {\bf Wikitext:}  The wikitext103 dataset from \cite{merity2016pointer}, which consists of 103 million words from English Wikipedia articles.
\end{itemize}
Size statistics are summarized in Table~\ref{tab:datasets}.

While Wikitext and CCNews are factual, Books is fiction and comprises a wide variety of writing styles.
The CCNews corpus has the narrowest domain and it is two orders of magnitude larger than Wikitext.
Overall, these datasets are interesting because they enable us to assess the ability of discriminators 
to fit and generalize across various axes, from the amount of data available at training time
to the richness of style and relatedness among the different  data sources.

On Wikitext and Books, we extract positive sequences from windows of text that are 160 tokens long with a stride of 40.
On the larger CCNews we do the same except that we stride by 160 tokens.
This protocol to mine positives is used both at training and test time, although at test time we limit the evaluation to
60,000 randomly chosen positive samples.

We use a Byte Pair Encoding (BPE)~\citep{sennrich2015neural} in order to represent all the dataset with a common vocabulary.
In particular, we use the byte level BPE vocabulary introduced by~\citet{radford2019language}, which contains 50k tokens.

\begin{savenotes}
\begin{table}[t]
\center
\begin{tabular}{lcccccccc}
\toprule
& \multicolumn{8}{c}{\bf Generators} \\
&\multirow{2}{*}{Conv} &\multirow{2}{*}{TransfSmall} &\multirow{2}{*}{TransfBig} & \multirow{2}{*}{TransfHuge} & \multicolumn{4}{c}{Pre-trained GPT-2} \\
\cmidrule{6-9}
&  &  &  &  & small & med & large & huge \\
\midrule
embed. & 13 & 26  &  51 & 77   & 39  &  52 & - & - \\
others & 164 & 19 & 151 & 1360 & 97  & 327 & - & - \\
total & 176 & 45  & 203 & 1437 & 137\footnote{We use models from the HuggingFace repository at \url{https://github.com/huggingface/transformers},
and report here the sizes of these models as they were used to generate data for Table~\ref{tab:unconditinal}.
Note that the OpenAI GPT-2 repository at \url{https://github.com/openai/gpt-2} defines models sizes as 124M and 355M for small and medium model correspondingly.\label{fn:repeat}}
& 380\footref{fn:repeat}  & 762\footnote{As reported in~\citet{radford2019language}.\label{fn:gpt2_paper}} & 1542\footref{fn:gpt2_paper} \\
\bottomrule
 \end{tabular}
\caption{\small Number of parameters (in millions) for the generator language models.
The computational cost is directly related to the number of parameters in other layers than the input embedding layer (second row).}
\label{tbl:generation_models_size}
\end{table}
\end{savenotes}

\begin{table}[t]
\center
\begin{tabular}{lccccc}
\toprule
&  \multicolumn{5}{c}{\bf Discriminators} \\
&  Linear & BiLSTM & BiLSTM Big &  UniT & BiT \\
\midrule
embed. &        0.1  & 26 & 39  &  51  & 51 \\
others &   0         & 23 & 90  &  151 & 304 \\
total &          0.1 & 49 & 129 &  203 & 355 \\
\bottomrule
\end{tabular}
\caption{\small Number of parameters in millions for the discriminator.
The computational cost is directly related to the number of parameters in other layers than the input embedding layer (second row).
\label{tbl:scoring_models_size}}
\end{table}

\subsection{Generator Architectures} \label{sec:generarch}
We mainly use a Transformer-based network~\citep{vaswani2017attention} to generate negatives.
We have a medium, large, and huge Transformer model based on the
architecture used in~\citet{baevski2018adaptive}, yielding three language
models in total: TransfSmall, TransfBig and TransfHuge; see details also in Table~\ref{tbl:generation_models_size}.

The small models use 6 blocks each containing a multi-head attention module with 8 heads. The large
models use 12 blocks each containing a multi-head attention module with 16 heads.
The huge models use 48 blocks each containing a multi-head attention module with 25 heads.
These Transformer
models are also implemented in~\citet{ott2019fairseq} as \texttt{transformer\_lm}, \texttt{transformer\_lm\_big},
and \texttt{transformer\_lm\_gpt2\_big}.
The TransfHuge has 10$\times$ the number of parameters of TransfBig and it is trained on CCNews only.
For each architecture except for TransfHuge we train two models on each each dataset: left-to-right and right-to-left.

In addition to the Transformer generator, we also  consider a 12-layer convolutional architecture
(Conv)~\citep{dauphin2017language}, and we also use the third-party trained GPT-2 models~\citep{radford2019language} as described
in Section~\ref{sec:cross-corpus}.

For the generalization study, we use these language models to generate either a prefix or a suffix, while for the text generation experiments we only consider generation conditioned on
a prefix. Unless otherwise specified, the context is either $120$ or $140$ tokens long (with equal probability).
Positive and negative examples have $40$ or $20$ tokens depending on the context size, for an overall length of $160$ tokens in
all cases. In preliminary experiments, we found that increasing the size of the generations and reducing the size of the context
makes the learning task significantly easier. We analyze the effect of the context size in Section~\ref{sec:analysis}.

We sample from each model's conditional distribution with a temperature of $1$. We do not consider sampling with beam search, as this tends to produce degenerate samples that would be easily detected~\citep{holtzman2019curious}.

\subsection{Discriminator Architectures} \label{sec:expscorer}
We consider three architectures for the discriminators:
\smallskip
\newline
\noindent
{\bf Linear} which computes a score via a bag of tokens: $f(w_1, \ldots , w_n)  = \left(\sum_{i=1}^n u_{w_i}\right)$,
where $u_i$ is a learned scalar parameter corresponding to the $i$-th token in the vocabulary.
\smallskip
\newline
\noindent
{\bf BiLSTM}~\citep{birnn,bilstm} which computes a score through $L$ bidirectional layers using LSTM recurrent units~\citep{lstm}, as
 $\mbox{Linear}(\mbox{AvgPool}(h_{L,1},\dots,h_{L,n}))$, where $h_{L,i}$ is the hidden state at position $i$ and layer $L$ which is the concatenation of the forward and
backward hidden states, AvgPool averages hidden states over positions and Linear uses a vector of parameters to project the hidden state down to a scalar value.
We consider two versions, referred to as ``BiLSTMsmall'' and ``BiLSTMbig''.
Both have 4 layers, but BiLSTMsmall has 512 units in both the embedding layer and the hidden layers, while
BiLSTMbig has 758 units in the embedding layer and 2014 units in the  hidden states.
\smallskip
\newline
\noindent
{\bf Transformer}~\citep{vaswani2017attention,bert} which computes a score similarly to the BiLSTM's,
except that each bi-LSTM layer is replaced
by a either a bidirectional Transformer layer (BiT), or a Transformer with causal self-attention (UniT).
For unidirectional models we use the same averaging technique as with BiLSTM models.
For bidirectional models the score is computed via:
$f(w_1, \ldots , w_n)  = u^\top h_{L,1} + b$, where $h_{L,1}$ is the top layer hidden state at the first position (as common practice also in prior work~\citep{bert}).
BiT uses the BERT-Large architecture~\citep{bert} initialized from~\citet{liu2019roberta}.
It uses 24 self-attention layers with 1024 units and 16-head attention each.
UniT has instead 12 layers with 1024 units and 16 attention heads per layer and it is initialized from a language modeling
task as in~\citet{radford2019language}.

For all models, we use Adam~\citep{kingma2014adam} optimizer with warmup.
We use data-parallel synchronous multi-GPU training with up to 24 nodes, each with 8 Nvidia V100 GPUs.
To improve training speed, we use mixed precision training\footnote{\url{https://github.com/NVIDIA/apex}}.
Following common practice we clip the norm of the gradient vector~\citep{pascanu2013difficulty}.
More details about hyper-parameter setting can be found in Appendix Table~\ref{tbl:hyperparams},
while Table~\ref{tbl:scoring_models_size} reports the number of parameters of each classifier.

\subsection{In-domain Generalization} \label{sec:in_domain}
In Table~\ref{tab:in_domain} we report the results of an in-domain generalization experiment using our large language model, TransfBig.
In this experiment, at test time the discriminator receives negatives generated by a generator language model that has the same architecture and that has been
trained on the same training data as the training generator, but with a different random seed (and of course we use prefixes and ground truth examples from the test set at test time). 

We observe that when the discriminators have similar representational power compared with the generator
(UniT, see Table~\ref{tbl:scoring_models_size}),
they are able to distinguish real from fake completions fairly accurately, reaching an accuracy of more than 90\%
on the Books dataset (which is easier since it exhibits the larger variety of style and topics\footnote{According to Table \ref{tab:main}, our best baseline BALM-24L achieved a perplexity of 13.92 on CC-News, but 18.24 on Book Corpus. Even using our best model, the perplexity we got on CC-News is 12.10-12.16, while on Book Corpus it's 15.17-15.22.}),
and attaining above 88\% on the more challenging CCNews dataset (for which generation is easier and hence
discrimination harder). The Wikitext dataset has lower accuracy because the diversity of texts is higher and the dataset is smaller than the others.

Weaker discriminators are able to do comparably or better at discriminating real from fake than
the training generator itself used as a discriminator by taking the negative log probability of the sequence as a score.
Notably, this observation may not hold for all sampling strategies. For example, the negative log probability of sequences generated via beam search are significantly higher than real sequences~\citep{holtzman2019curious}; thus, such samples would be easily detected using the training generator as a discriminator.

We conclude that since a discriminator can easily tell if a piece of text contains machine-generated tokens, it should also be possible to use the discriminator score to improve the original
text generation method -- a topic we explore in Section~\ref{sec:ebm4text}. The reader interested in text generation can safely skip the next
section and directly dive into Section~\ref{sec:ebm4text} to resume discussion on how to leverage these discriminators for text generation. 

\begin{table}[t]
\vspace{-.5cm}
  \center
\begin{tabular}{lrrr}
\toprule
 &  Books &  CCNews &  Wiki \\
\midrule
Linear                    &            59.8 &                 58.6 &            56.3 \\
BiLSTMsmall                    &            84.7 &                 77.6 &            71.0 \\
BiLSTMbig                &            86.7 &                 80.1 &            72.8 \\
UniT                 &            91.7 &                 88.4 &            76.4 \\
\midrule
{\em TransfBig (log-likelihood)} &            57.1 &                 50.8 &            50.5 \\
\bottomrule
\end{tabular}
\caption{\small ``In domain'' generalization accuracy of discriminators (each row) on various text corpora.
A column corresponds to the corpus used to get positives and to fit the train and test language models,
which are TransfBig (Section~\ref{sec:generarch})
with different initial seeds.  The last row is the accuracy when using as score the negative log-probability of the
training generator over the whole sequence.
\label{tab:in_domain}}
\end{table}

\subsection{Application: Real/Fake Discrimination} \label{sec:eval_general}
Classifying if a document contains machine-generated text is an interesting application on its own. In the previous section we have seen that discriminators do pretty well at detecting text generated by language models which have the same architecture and when trained on the same data used in training the generator. In practice however, the designer of the real/fake text detection system has access to neither the architecture nor
the data used by the (test) adversary language model. In this section, we then study to which extent the discriminator generalizes to these more extreme, but also more realistic, conditions.
\begin{table}[t]
\center
\begin{tabular}{lcc}
\toprule
& \textsc{corpus:} & \textsc{generator architecture:} \\
& $C_{\mbox{train}} = C_{\mbox{test}}$ & $A_{\mbox{train}} = A_{\mbox{test}}$ \\
\midrule
in-domain & \cmark & \cmark \\
cross-architecture & \cmark & \xmark \\
cross-corpus & \xmark & \cmark \\
wild & \xmark & \xmark \\
\bottomrule
\end{tabular}
\caption{\small Four evaluation settings considered in this work, described in Section~\ref{sec:eval_general}.
\label{tab:eval_setup}}
\end{table}

We test the ability of the ``discriminator'' to generalize by evaluating how well it performs against text produced by generators that have different architectures than the generators it sees at training time and/or that are trained on different corpora.

More formally, let $C_{\mbox{train}}$ be the corpus used to train the generator $G_{\mbox{train}}$ which in turn produces negatives for {\em training} the discriminator.
$G_{\mbox{train}}$ has architecture $A_{\mbox{train}}$. Finally, let $C_{\mbox{test}}$ be the corpus used to train the generator  $G_{\mbox{test}}$ which in turn produces negatives
to {\em test} the discriminator. We denote by $A_{\mbox{test}}$ the architecture of $G_{\mbox{test}}$.

Note that $G_{\mbox{train}}\neq G_{\mbox{test}}$ even if $A_{\mbox{test}}=A_{\mbox{train}}$ and $C_{\mbox{train}} = C_{\mbox{test}}$, as we use different random seeds.
Moreover, note that each corpus has distinct training and test parts.
As a result, even when $C_{\mbox{train}} = C_{\mbox{test}}$, the discriminator is tested using positives and negatives derived from the test part of $C_{\mbox{test}}$, meaning that the positive is
a sequence extracted from the test set and the negative is produced by the generator conditioned on an affix taken from the test set.
Finally, when $C_{\mbox{train}} \neq C_{\mbox{test}}$ the discriminator is tested using both positives and negatives derived from $C_{\mbox{test}}$.

We consider four settings, as shown in Tab.~\ref{tab:eval_setup}:
\begin{itemize}
\item
In the {\bf in-domain} setting,  $C_{\mbox{test}}$ is the same as $C_{\mbox{train}}$ and $A_{\mbox{test}}=A_{\mbox{train}}$; this has already been discussed in Section~\ref{sec:in_domain}.
\item
In the {\bf cross-architecture} setting, again $C_{\mbox{test}}$ is $C_{\mbox{train}}$, but $A_{\mbox{test}}$ is different from $A_{\mbox{train}}$. For instance, $A_{\mbox{test}}$ could be
a Transformer while $A_{\mbox{train}}$ could be a convolutional architecture.
\item
In the {\bf cross-corpus} setting, $A_{\mbox{test}} =A_{\mbox{train}}$ but $C_{\mbox{test}}$ is different than $C_{\mbox{train}}$,
and $G_{\mbox{test}}$ is trained on the training split of $C_{\mbox{test}}$, while $G_{\mbox{train}}$ trained on the training split of $C_{\mbox{train}}$. For instance, $C_{\mbox{train}}$ could be
a dataset extracted from Wikipedia while $C_{\mbox{test}}$ could be a dataset of news.
\item
In the {\bf wild} setting, both $C_{\mbox{test}}$ is different than $C_{\mbox{train}}$ and $A_{\mbox{test}}$ is different from $A_{\mbox{train}}$.
\end{itemize}
In all settings, we report performance in terms of average classification accuracy balancing the positive and negative classes equally.

\subsubsection{Cross-Architecture Generalization}  \label{sec:cross-arch}
\begin{table}[t]
\vspace{-0.2cm}
\begin{tabular}{lrrr}
\toprule
{} &  Conv &  TransfSmall & Mean  \\
\midrule
Conv &                    92.9 &            81.2 & 87.1\\
TransfSmall         &                    86.5 &            87.9 & 87.2 \\
\bottomrule
\end{tabular}
\caption{\small Cross-architecture generalization accuracy using the Wikitext dataset for both training
and testing ($C_{\mbox{train}} = C_{\mbox{test}}$).
Each row is a model architecture used for generating the training negatives ($A_{\mbox{train}}$), and
each column is a model architecture for generating the testing negatives ($A_{\mbox{test}}$).
The discriminator is UniT.}
\label{tab:cross_architecture}
\end{table}
In Table \ref{tab:cross_architecture},
 we assess how well the UniT discriminator
generalizes to different generator architectures at test time, namely Conv and TransfSmall.
As a reference on the Wikitext dataset,
the test perplexity of Conv and TransfSmall are
 35.4 and 33.5, respectively. Therefore, these two generators attain roughly the same perplexity, despite
Conv having about 4 times more parameters, see Table~\ref{tbl:generation_models_size}.

Surprisingly, UniT has significantly harder time discriminating
TransfSmall negatives with an in-domain rate of 87.9\%, compared to 92.9\% of  Conv.
Besides, UniT  trained with TransfSmall negatives is more robust to the (weaker) Conv generations, than vice versa (trained with Conv generations, test with TransfSmall generations),
with a mild 1.4\% accuracy drop.
Lastly, if we average
values across columns for each row, we see that when tested with mixed negatives, the discriminator is slightly more accurate
when trained with the harder negatives produced by TransfSmall compared to those produced by Conv.

\subsubsection{Cross-Corpus Generalization}\label{sec:cross-corpus}
In Table~\ref{tab:cross-corpus} we show the results of generalizing across corpora using UniT as a discriminator
and TransfBig as generator both at training and test time.
We observe that models generalize less well across corpora;
for instance, when testing on Wikitext a discriminator 
trained with either Books or CCNews, the accuracy is 59.1\% and 65.5\%, respectively.
However, training on the union of two of the corpora
gives a large improvement over training on just one or the other when tested on the third.

Finally, training on the union of {\em all} the three corpora (last two rows)
yields a discriminator that is very robust to the testing conditions, with an accuracy which is on par if not better
than training on in-domain data, even for the largest CC-News dataset (second column).

We also tested the bidirectional Transformer discriminator BiT with 355M parameters (almost twice as UniT) and with cloze-driven pretraining,
 and found that on CC-News it improves accuracy by more than 5\% when it is trained on the union of all corpora.
As BiT was pretrained using the whole Wikipedia rather than the training part of Wikitext103,
we do not report its accuracy on Wikitext103.

\begin{table*}[t]
\center
\begin{tabular}{lccc}
\toprule
\multirow{2}{*}{\textsc{train corpora}} \quad \, & \multicolumn{3}{c}{\textsc{test corpora}} \\
                &  Books            &  CCNews                &  Wiki             \\
\midrule
Wiki                &            70.9 &                 73.6 &            76.4 \\
\hline
Books               &            91.7 &                 63.5 &            59.1 \\
Books + Wiki        &            91.5 &                 73.6 &            78.3 \\
\hline
CCNews               &            60.6 &                 88.4 &            65.5 \\
Books + CCNews       &            90.4 &                 88.5 &            68.3 \\
CCNews + Wiki        &            73.5 &                 88.3 &            81.0 \\
\hline
ALL (UniT)      &            90.4 &                 88.5 &            80.9 \\
ALL (BiT)       &            94.1 &                 94.1 &             - \\
\bottomrule
\end{tabular}
\caption{\small Cross-corpora generalization accuracy using TransfBig generator and UniT discriminator (except for the last
row which used a bidirectional Transformer).
Each row specifies the corpora used at training time, $C_{\mbox{train}}$. Each column shows the corpus used at test time, $C_{\mbox{test}}$.
\label{tab:cross-corpus}}
\end{table*}

\subsubsection{Generalization in the Wild}
\begin{table}[t]
\center
  \begin{tabular}{lclll}
  \toprule
    \multicolumn{1}{r}{Discriminator $\rightarrow$ } & TF-IDF$^*$ & \multicolumn{3}{c}{BiT} \\
    \cmidrule{3-5}
    \multicolumn{1}{r}{Test setting $\rightarrow$ } & in-domain & in-domain & cross-architecture & wild  \\
    \midrule
 Small (137) top-k      & 96.79  & 99.09  {\tiny(99.3)} &  -      &  93.25 \\
 Small (137) temp=1    & 88.29  & 99.80                 &  -       &  66.04 \\
 Med   (380) top-k      & 95.22  & 98.07 {\tiny(98.5)} &  97.37  {\tiny(96.6)} &  88.19 \\
 Med   (380) temp=1    & 88.94  & 99.43                &  97.35   &  55.06 \\
 Big   (762) top-k      & 94.43  & 96.50 {\tiny(97.9)} &  93.58  {\tiny(90.9)} &  83.88 \\
 Big   (762) temp=1   & 77.16  & 99.42                 &  95.96  &  64.03 \\
 Huge  (1542) top-k      & 94.43  & 95.01  {\tiny(96.0)} &  90.17 {\tiny(79.3)} &  79.18 \\
 Huge  (1542) temp=1    & 77.31  & 99.00               &  91.76  &  61.29 \\
\bottomrule
  \end{tabular}
  \caption{\small
    \label{tab:unconditinal}
    Generalization {\em in the wild} of the discriminator to {\em unconditional} generations from various GPT-2 models (model size in
parentheses, followed by sampling method used). Each row contains the accuracy on the corresponding test set.
TF-IDF results are taken from~\citet{gpt2gen}. Results in parentheses are taken from \url{https://openai.com/blog/gpt-2-1-5b-release/}.}
\end{table}

We now consider a BiT discriminator trained on the union of all the three datasets (Wiki, Books and CCNews) using TransfBig generations at training time,
and investigate its generalization when tested both on a new domain, WebText, and on negatives produced by a new architecture, GPT-2.
This test dataset~\citep{gpt2gen} has a 250,000 generated texts with either top-k sampling or sampling with temperature 1. Empirically, top-k sampling produces more diverse samples than beam search, while avoiding unlikely sequences \citep{fan18}.

To adapt our fixed-length discriminator to this task we simply split the text segments into non-overlapping blocks of 160 tokens.
During finetuning we treat all blocks in a set as either positives or negatives.
During evaluation we take the mean prediction over all blocks in a segment as a prediction for the whole segment.
Finally, since this discrimination task is unconditional, we train our discriminator on all possible prefixes including an empty prefix.

As a first baseline of comparison, we report in-domain accuracy comparing the discriminator to a TF-IDF baseline provided by~\citet{gpt2gen},
see Table~\ref{tab:unconditinal}.
In this case, we finetune the discriminator on the training set of each of the datasets, following the same protocol
used by the provided TF-IDF baseline. We notice that BiT discriminator has consistently better performance, with an accuracy greater than 95\%.

As an additional baseline, we compute the cross-architecture performance by
 finetuning the discriminator only on the generations from the small GPT-2 model (both top-k and random sampling),
and applying the model to the other datasets. We observe that the discriminator still generalizes remarkably well in this setting.
In particular, we can outperform the in-domain TF-IDF baseline when the generator is less than
three times bigger than what was used at training time.
Comparing to the results reported by the creators of the dataset, we observe that our discriminator generalizes better even though it performs
a little worse in the in-domain setting.

Finally, in the {\em wild} setting we explore generalization to a black-box generator where the discriminator is trained only on out of domain corpora without any finetuning on WebText, see last column of Table~\ref{tab:unconditinal}.
While the discriminator still works much better than a random predictor, it lags behind the simple (in-domain) TF-IDF baseline.
This suggests that matching the domain of the training set is more important than matching the model complexity.

\section{Improving Text Generation with Energy-Based Models} \label{sec:ebm4text}
In the previous sections, we checked empirically that machine-generated text by current state-of-the-art locally normalized and auto-regressive language models can be easily discriminated, albeit to a lesser extent
in extreme generalization conditions. In this section, we then investigate how such classifier scores can be integrated into the original language model in order to improve its generation quality.

Towards this end, we are going to consider energy models (or ``scoring functions'',  or ``discriminators''), $E = E_{\theta}(x_1,\ldots , x_T)$.  These take in the {\em entire sequence at once}, 
and need not be normalized or produce probabilities over the sequence.  The models $P_{\phi}$ of Section~\ref{sec:basicsetup} are a special case of the models $E_{\theta}$.

Next, we will show that a particular form of such energy-based models, namely a residual formulation, enables us to make an efficient and straightforward use of classifier scores to improve generation. 

\subsection{Training an Energy Model}
At a high level, to train an energy model, we need to decrease the energy of sequences that come from the true distribution (``positives''), and increase the energy of sequences that are not from the true distribution (``negatives'').    The positives are taken to be the training data; and the challenge is to efficiently find good negatives.   The method used for mining negatives will depend (amongst other factors) on the loss function for the energy model and the particulars of the data space. 

There is a clear trade-off between the computation cost of finding negatives and the quality of negatives. For instance, setting negatives to random sequences is very cheap but one would need to sample many times before encountering negatives that receive low energy and are somewhat close to the real data. Conversely, many approaches to negative mining use optimization to find negatives that are erroneously assigned low energy by the $E_{\theta}$.   While these can be effective (e.g. \cite{mordatch19}), they are often time consuming, even in the continuous case where gradient-based optimization methods can be used.  In the text (discrete) setting, the situation becomes worse, as the optimization problem becomes combinatorial. 

An important simplification in this work, which is similar to \citet{wang2018learning,parshakova2019global}, is to use a base pretrained auto-regressive language model as the source of negatives, such that the energy model operates in the ``residual'' of the base model.  In particular, this simplifies searching for negatives, as these can be taken to be generations from the base LM $P_{\phi}$.   

That is, we take the generative model to be:
\begin{equation}
P_\theta(x_{p+1},\ldots , x_T|x_1,\ldots , x_p) =
\frac{ P_{\phi}(x_{p+1},\ldots , x_T|x_1,\ldots , x_p)
\exp (- E_\theta(x_1,\ldots , x_T))}{Z_\theta(x_1,\ldots , x_p)}
    \label{eq:gen}
\end{equation}
where $Z_\theta(x_1,\ldots , x_p)$ is a normalizing factor known as {\em partition function}, $\phi$ are a fixed set of parameters and $\theta$ are the parameters subject to learning.
Computing the partition function is intractable
in our case since it involves a sum over $|V|^{T-p}$ terms which grow exponentially with the sequence length: in our
experiments the size of the vocabulary is 50,096 and the length of the generation is 40 tokens (the length of the prefix is 120 tokens).
We call $P_\theta$ the {\em joint} model, and $E_\theta$ the {\em residual energy function} since $P_{\phi}$ is fixed throughout training. The goal of training is to learn the parameters
of the energy function such that the joint model distribution gets close to the data distribution.
For the sake of reducing clutter in the notation, we will drop the conditioning variables in the following discussion.

We train our residual energy function using Noise Contrastive Estimation (NCE)~\citep{gutmann2010noise}, and
more specifically its conditional version~\citep{ma18}. NCE requires two distributions: the model distribution and the noise
distribution. In our case, the model distribution is the joint model of Eq.~\ref{eq:gen}, $P_\theta$, while the noise distribution
is the pretrained language model, $P_{\phi}$. NCE then trains a binary classifier on the difference of log-probability scores
of these two models. Since our joint model is the product of the energy function (whose parameters we want to learn) with
$P_{\phi}$, the difference reduces to: $\log P_\theta - \log P_{\phi} = -E_{\theta}$ (the partition function is omitted here because it's not part of the model: it is implicitly induced if we want to get normalized probabilities). Therefore, under these modeling
assumptions of residual learning and noise model, the objective function becomes:
\begin{equation}
    \max_{\theta} \mathbb{E}_{x_+ \sim P_{data}} \log \frac{1}{1+\exp(E_\theta(x_+))} +  \mathbb{E}_{x_- \sim P_{\phi}} \log \frac{1}{1+\exp(-E_\theta(x_-))}. \label{eq:nce}
\end{equation}
Notice how this is precisely the loss function and model introduced in Eq.~\ref{eq:loss_realfake} with the change of variable  
$s_{\theta} = - E_{\theta}$. 
Therefore, training a real/fake discriminator also amounts
to estimating the model parameters of a density estimator operating on entire sequences!

It can be shown that  if $P_{\phi}$ has the same support as $P_{data}$, then the objective function in Eq.~\ref{eq:nce} reaches its maximum at $\log P_{\phi}(x) -E_\theta(x) = \log P_{data}$, if there exists such $\theta$; that is, the optimum of the above objective is reached at data
distribution with infinite amount of data and model with enough capacity\footnote{Note that an auto-regressive probabilistic LM will also reach its optimum at $P_{data}$ given enough capacity, but in the residual energy model, there's no capacity requirement on the base language model $P_{\phi}$.}.  This follows from the proof in \citet{gutmann2010noise}, and is also proved in
\citet{ma18}\footnote{From \citet{ma18} Assumption 2, for conditional NCE the model needs to be flexible enough such that
the self-normalizing property can be satisfied conditioned on any prefix.}.  Note that at optimum, $P_{\phi}(x) \exp(-E_\theta(x))$ is self-normalizing: instead of $P_{\theta}(x) \propto P_{\phi}(x) \exp(-E_\theta(x))$, we have $P_{\theta}(x) = P_{\phi}(x) \exp(-E_\theta(x))$. However, at evaluation time we still need to estimate the partition function, since we cannot guarantee that this optimum can be reached after training.

\section{\label{sec:rebm_metrics}Evaluation Tasks and Metrics for Text Generation}
In this section we discuss how to evaluate residual EBMs. Since these models estimate probabilities of text, we first describe how to estimate perplexity, a common metric for the language modeling task. 
We then show how to use residual EBMs to generate text.

\subsection{Perplexity} \label{sec:eval_ppl}
A commonly used protocol for evaluating generative models of sequences, especially language models, is perplexity (PPL), which is
equal to  $2^{- \frac{1}{T-p} \sum_{i=p+1}^T \log_2 P(x_i | x_{i-1}, \cdots, x_1)}$. PPL can be interpreted as the average number  of
tokens the  model is uncertain of at every time step.
 Since the log-likelihood required by PPL relies on estimating the partition function $Z_\theta = \sum_x P_{\phi}(x) \exp (-E_\theta(x))= \mathbb{E}_{x\sim P_{\phi}}\exp (-E_\theta(x))$, we derive two estimators for the log-partition function $\log Z_\theta$ based on the work of \citet{nowozin2018debiasing}.

\begin{theorem}
Denote $T_n$ as the empirical estimate of $\log \mathbb{E}_{x\sim P_{\phi}} \exp(-E(x))$ with $n$ samples $x^i\sim P_{\phi}$ $(i=1,\cdots, n)$, i.e., $T_n = \log \frac{1}{n}\sum_{i=1}^n\exp(-E(x^i)) $, then $\forall \epsilon > 0$, $\exists N>0$ such that $\forall n > N$ we have
\begin{equation}
      Z_\theta -\epsilon < \mathbb{E}[T_n] < Z_\theta  < \mathbb{E}[(2n-1) T_n - 2(n-1) T_{n-1}] < Z_\theta +\epsilon
\label{eq:bounds}
\end{equation}
\end{theorem}
The proof is given in Appendix~\ref{appendix:proof}. 

We can use the above two estimators ($T_n$ and $(2n-1) T_n - 2(n-1) T_{n-1}$) to estimate the lower and upper bounds of the partition function, but we want to emphasize that they are true only asymptotically (when $n$ is sufficiently large). Furthermore, the samples we used to estimate $T_{n-1}$ can be derived from those we used for estimating $T_n$: we leave one sample out at a time, and use the rest $n-1$ samples to get an estimate of $T_{n-1}$. We take the average of $n$ such estimates as the final estimate of $T_{n-1}$.
See \citet{nowozin2018debiasing} for implementation details and methods for improving numeric stability.

Similar to locally normalized models, we can also factorize the probabilities of an entire sequence step by step, as $P(x) = \prod_{t=1}^T P(x_t|x_{<t})$, and evaluate the PPL for each generation step. By marginalizing over the future, we can derive the following
per step probabilities:
\begin{equation}
P(x_t | x_{<t}) = P_{\phi}(x_t|x_{<t}) \frac{\mathbb{E}_{x_{t+1}',\cdots, x_{T}' \sim P_{\phi}(\cdot | x_{\le t})}[\exp (-E_\theta(x_{\le t}, x_{t+1}',\cdots, x_T'))]}{\mathbb{E}_{x_{t}',\cdots, x_{T}' \sim P_{\phi}(\cdot | x_{\le t-1})}[\exp (-E_\theta(x_{\le t-1}, x_{t}',\cdots, x_T'))]}.
    \label{eq:stepppl}
\end{equation}

The derivation of Eq.~\ref{eq:stepppl} can be found in Appendix~\ref{app:proof6}. 
The base probability $P_{\phi}(x_t|x_{<t})$  is adjusted by the expected negative energy of sequences starting with $x_1, \cdots, x_t$ (compared to using a random draw $x_t'$ from the base distribution in the denominator), such that those $x_t$'s that lead to better sequences (as judged by the energy model) are up-weighted. 
Since the summation involves exponentially many terms, unless
$t=T$, this is approximated  by sampling. Since both the numerator and the denominator take the same form as the partition function, we can use Eq.~\ref{eq:bounds} to estimate the upper and lower bounds. For example, the lower bound of $\log P(x_t | x_{<t})$ can be obtained by using the lower bound of the numerator and the upper bound of the denominator.

For $t=T$, we can calculate the log  probability by exhaustive enumeration. This gives us an idea of the true performance
of our model at the last step, and it also provides a sanity-check of the tightness of our estimators.

\begin{algorithm}[tb]
    \caption{Top-k Joint Sampling}
    \label{algo:main}
\begin{algorithmic}
    \STATE {\bfseries Input:} number of samples $n$ drawn from $P_{\phi}$, value of $k$ in top-k

    \vspace{0.1cm}
    \tcp{Get a set of samples from $P_{\phi}$}
    \tcp{Each $x^i$ is a full sequence}
    \STATE sample $n$ samples $\{x^1, \cdots, x^n\}$ from $P_{\phi}$ with top-k sampling
    \STATE calculate negative energies $s^i = - E_\theta(x^i)$ for each $x^i \in \{x^1, \cdots, x^n\}$
    \vspace{0.1cm}
    
    \tcp{Resample from the set of LM samples}
    \STATE sample $x=x^i$ with probability $\frac{\exp (s^i)}{\sum_{j=1}^n \exp(s^j)}$
    \STATE {\bfseries return} $x$
\end{algorithmic}
\end{algorithm}

\subsection{Generation}
Generating from the joint model is a non-trivial task. A naive way is to generate from the joint model auto-regressively, by marginalizing the future as in Eq.~\ref{eq:stepppl}, which we term \textbf{Top-k auto-regressive sampling}.
However, doing so is computationally expensive and impractical, and we only use this method for a qualitative analysis of the joint model in Appendix~\ref{appendix:qual}.

In order to generate efficiently, we use self-normalizing importance sampling~\citep{mcbook,grover2019bias}.
Under the assumptions
that the model from which we wish to  draw samples is the joint model, which is the product of the auto-regressive
model and the energy  function, and that the proposal distribution is the auto-regressive model itself,
sampling proceeds simply by:  a) sampling from the auto-regressive language  model, followed by b) resampling according
to the energy function.
 The algorithm is shown in Algorithm~\ref{algo:main}, where we introduce an optional top-k constraint \citep{holtzman2019curious} on the pretrained language
model to improve the quality of samples in the set.\footnote{Adapting to other types of local constraints
such as nucleus sampling~\citep{holtzman2019curious} is straightforward.}
Without the top-k constraint, as the number of samples goes to infinity, we would recover exact samples
from the joint model distribution.

In order to evaluate the quality of generations, we perform A/B testing using human raters. See Section~\ref{sec:eval_generations} for more details.

\section{Experiments Using Residual EBMs}
In this section, we empirically assess whether residual EBMs actually improve the baseline language model both in terms of perplexity scores and human evaluation of perceived generation quality.

    \subsection{Evaluating Language Modeling}  \label{sec:exp_generation}
When evaluating the EBM in terms of perplexity, we use a similar setting as before, except that we only condition on prefixes of size $120$, for
a total sequence length equal to $160$.

\noindent
\textit{Baselines.}\enskip
We consider as base language model (\textsc{Base LM}) used to generate negatives for the residual EBM, a Transformer language
model, which is also our first baseline model.

The {\em joint} model has as many parameters as the sum of the number of parameters in the base LM
and the number of parameters in the energy network.
To make a fair comparison, we consider two additional baselines that have the same number of parameters
as our joint model.

The first baseline is a Residual Auto-regressive Language Model baseline (\textsc{RALM}):
\begin{multline}
    \label{eq:baseline} \log P_{RALM} (x_t | x_{<t}) = \log P_{\theta} (x_t | x_{<t}) + \log P_{\phi} (x_t | x_{<t}) \\
    - \log \sum_{x_t'}\exp(\log P_{\theta} (x_t' | x_{<t}) + \log P_{\phi} (x_t' | x_{<t})),
\end{multline}
 where $P_{\theta}$ takes the form of another auto-regressive language model. The parameters of $P_{\theta}$ are trained by
exact maximum likelihood training of $P_{RALM}$. In our experiments $P_\theta$ used the same network architecture as $P_\phi$, the only difference between $P_\theta$ and $P_\phi$ is that they have different parameters.

The second baseline is an auto-regressive language model of the same size of our joint model (sum of the base LM
and energy function parameters), we dub this model Big
Auto-regressive Language Model (\textsc{BALM}). \textsc{BALM} is trained by standard token level cross-entropy loss.

\noindent
\textit{Residual EBM Architecture.}\enskip We consider two versions: UniT and BiT as described in Section~\ref{sec:expscorer}.
UniT has the same architecture as $\textsc{Base LM}$, except for the additional top layer projecting the mean-pooled hidden states
to a scalar energy value. We initialize its parameters with a language model trained on the same dataset.
BiT has two variants, a $\textsc{BiT-Base}*$ following the architecture of RoBERTa-Base, and a $\textsc{BiT-Large}*$ following
RoBERTa-Large~\citep{liu2019roberta}. We initialize the parameters with a trained BERT, and we use $*$ to mark usage of
external data~\citep{liu2019roberta}, otherwise it means that pretraining was done only on the training set.
Notice how our model can be interpreted as a natural way to finetune large  pretrained bidirectional models for the
language modeling task.
Detailed hyper-parameter settings can be found in Appendix~\ref{app:hyperparams}.

\noindent
\textit{Results.}\enskip
In Table~\ref{tab:main} we compare models in terms of their perplexity.
In the upper part of the table, we can see that on both datasets, residual EBMs with causal attention \textsc{joint UniT} outperforms
the baseline \textsc{RALM} with approximately the same number of parameters. The non-residual baseline $\textsc{BALM}$
performs similarly to \textsc{joint UniT}, which might be due to the limitation that $P_{\phi}$ is not trained jointly with the
residual model in \textsc{joint UniT} (and $\textsc{RALM}$). However, by using our EBM approach,
we can remove the causal attention mask and use bidirectional models, which achieves better performance than both baselines and \textsc{joint UniT}:
without external data, \textsc{joint BiT-Base} reaches a higher performance than \textsc{joint UniT} with fewer parameters.

In the middle part of the table, we show that if we make the big language model baseline \textsc{BALM} 
deeper (\textsc{BALM-24L}) (24 layers instead of 12, for the same number of parameters) we attain lower perplexity.
However, training the joint model \textsc{Joint BiT-Base} on the residual of a deeper language model \textsc{Base LM-24L}
yields even lower perplexity, despite having fewer parameters. By using the same number of parameters as \textsc{BALM-24L}, \textsc{Joint Bit-Med} further decreases perplexity. 

Finally, in the lower part of the table, we show that our approach can leverage pretrained bidirectional Transformers to improve language modeling: by initializing from RoBERTa-Base, \textsc{joint BiT-Base*} obtains better perplexities than not using external data, both when the base language model has 12 layers, and when it has 24 layers.
By initializing from the state-of-the-art pretrained bidirectional Transformer
RoBERTa-Large, \textsc{Joint BiT-Large*} with a 24-layer base language model performs the best on both datasets.
\begin{table}[h]
    \centering
    \footnotesize
    \begin{tabular}{@{}lcccc@{}}
    \toprule
    \multirow{2}{*}{Model (\#parameters)}   & \multicolumn{2}{c}{CC-News}   & \multicolumn{2}{c}{Toronto Book Corpus}\\
            & Val   & Test                  & Val   & Test \\
            \midrule
    \textbf{Without External Data, 12 Layers}\\
    \ \ \textsc{base LM} (203M)     & 18.41 & 17.57                 & 16.16 & 18.29 \\
    \ \ \textsc{RALM} (LM+203M) & 17.01 & 16.17  & 15.71 & 17.85 \\
    \ \ \textsc{BALM} (408M) & $16.50$ & $15.74$ & $\mathbf{15.00}$ & $\mathbf{16.99}$  \\
    \ \ \textsc{joint UniT} (LM+203M) & 16.42-16.44 & 15.57-15.58 & 15.12-15.13 & 16.98-17.00\\
    \ \ \textsc{joint BiT-Base} (LM+125M) & $\mathbf{15.32}$-$\mathbf{15.35}$ & $\mathbf{14.61}$-$\mathbf{14.64}$ & -  & -  \\
    \midrule
    \textbf{Without External Data, 24 Layers}\\
    \ \ \textsc{Base LM-24L} (203M) & $15.71$ & $14.89$    &  $15.61$& $18.14$\\
    \ \ \textsc{RALM} (LM-24L+203M) & $15.70$ & $14.89$  & $15.63$ & $18.17$ \\
    \ \ \textsc{BALM-24L} (408M) & $14.58$ & $13.92$    & $15.20$ & $18.24$\\
    \ \ \textsc{joint UniT} (LM-24L+203M) & $14.59$-$14.61$ & $13.81$-$13.82$ & $\mathbf{15.12}-\mathbf{15.16}$ & $\mathbf{17.46}$-$\mathbf{17.48}$\\
    \ \ \textsc{joint BiT-Base} (LM-24L+125M) & ${13.68}$-${13.69}$ &   ${13.01}$-${13.03}$  & - & -\\
    \ \ \textsc{joint BiT-Med} (LM-24L+203M) & $\mathbf{12.97}$-$\mathbf{13.01}$  &  $\mathbf{12.38}$-$\mathbf{12.42}$   & - & -\\
    \midrule
    \textbf{With External Data}\\
    \ \ \textsc{joint BiT-Base*} (LM+125M) & $15.11$-$15.17$ & $14.37$-$14.42$    & $14.14$-$14.16$ & $15.72$-$15.74$\\
    \ \ \textsc{joint BiT-Large*} (LM+355M) & ${14.59}$-${14.61}$ & ${13.97}$-${14.00}$ & ${13.80}$-${13.83}$  & ${15.33}$-${15.36}$ \\
    \ \ \textsc{joint BiT-Base*} (LM-24L+125M) & $13.60$-$13.62$ &   $12.93$-$12.95$  & $14.11$-$14.12$ & $16.17$-$16.18$\\
    \ \ \textsc{joint BiT-Large*} (LM-24L+355M) & $\mathbf{12.71}$-$\mathbf{12.77}$  & $\mathbf{12.10}$-$\mathbf{12.16}$    & $\mathbf{13.30}$-$\mathbf{13.34}$ & $\mathbf{15.17}$-$\mathbf{15.22}$\\
    \bottomrule
    \end{tabular}
    \caption{Validation and test perplexity on CC-News and Toronto Book Corpus. * denotes models initialized with RoBERTa
trained on additional data. The joint model perplexity ranges are \textbf{estimated} using 100,000 samples, see
Eq.~\ref{eq:bounds}. The number of  parameters of each model is shown in parentheses.
}
    \label{tab:main}
\end{table}

One caveat of our evaluation protocol is that the perplexity bounds are only estimates, 
which might not reflect the true value, particularly since the number of possible sequences grows exponentially 
with the number of words that are generated. 
We therefore break down perplexity per position in the generated sequences as in Eq.~\ref{eq:stepppl}, 
and compare the estimated PPLs to the true enumerated 
PPLs at the last position, as shown in Figure~\ref{fig:main}. We find that at the final generation step, the estimated 
bounds agree remarkably well with the exact values, proving that our method at least gets a reasonable PPL estimate at the last generation step, and that $\textsc{Joint BiT-Med}$ and $\textsc{Joint BiT-Large*}$ outperforms baselines at the last generation step for sure. 

\begin{figure}[t]
  \centering
  \includegraphics[trim={3cm 0 0 0.7cm},clip,width=1.0\linewidth]{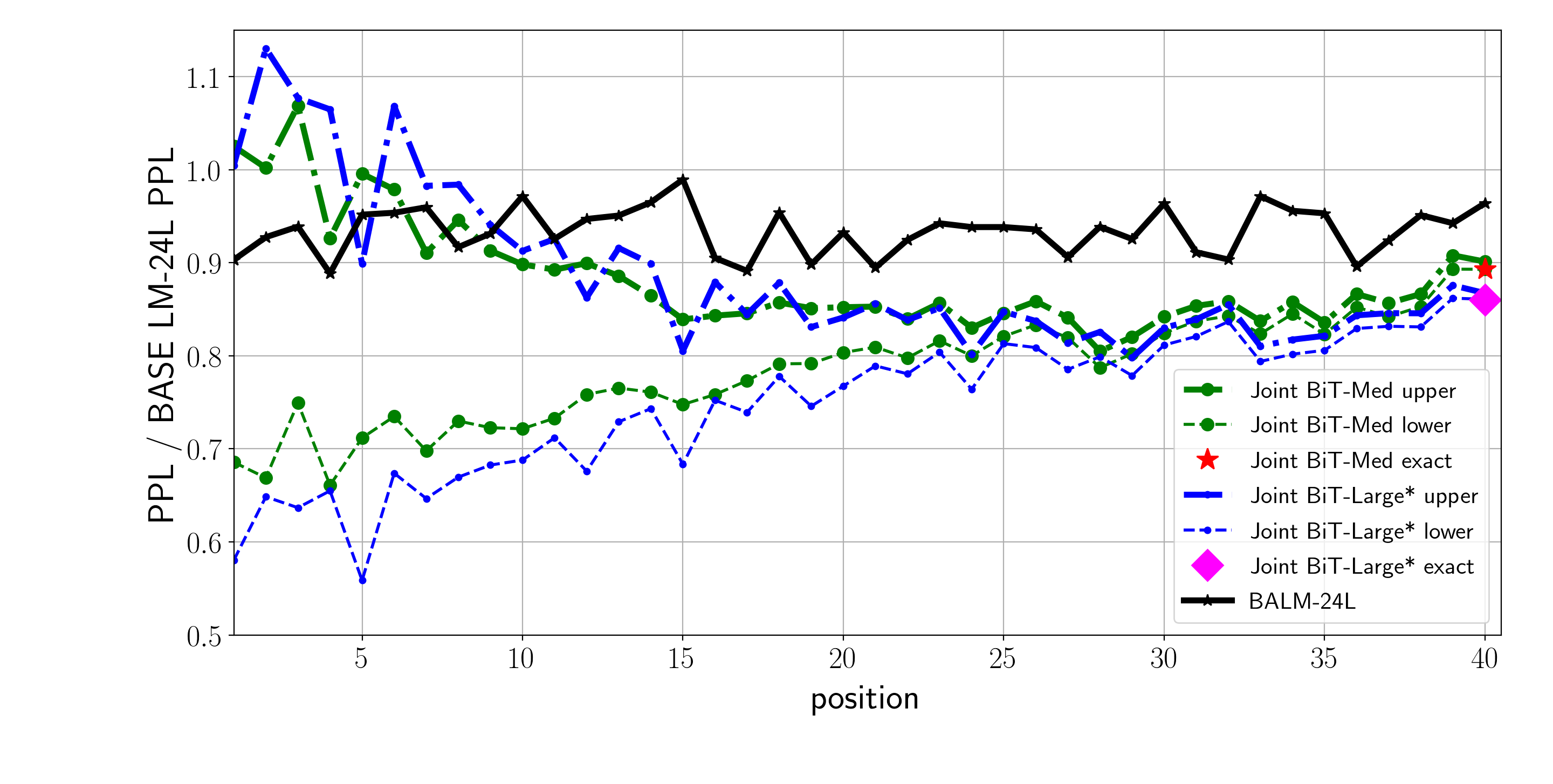}
\caption{\label{fig:main}Perplexity gain of $\textsc{Joint BiT-Med}$
  and $\textsc{Joint BiT-Large}*$ (using $\textsc{Base LM-24L}$) at
  each position relative to $\textsc{Base LM-24L}$
on the test set of CC-News. At each position the lower and upper
bounds (Eq.~\ref{eq:stepppl} estimated using the method in
Eq.~\ref{eq:bounds}, see Section~\ref{sec:eval_ppl} for more details) are estimated
using 20,000 samples. The shorter the horizon (moving to the right), the tighter the estimation is but also the more
limited the gains compared to base LM as un-normalized models are most useful on longer generations.
}
\end{figure}

\subsection{Evaluating Model Generations}\label{sec:eval_generations}
\begin{table}[!t]
    \centering
    \begin{tabular}{lclcr}
    \toprule
    Model1 (baseline) &    & Model2 (compared model) & Rate & p-value\\
    \midrule
    \textsc{base LM}& \multirow{9}{*}{$<$}  & \textsc{joint uniT} & 52.85\% &0.16 \\
    \textsc{base LM}&   & \textsc{joint BiT-Base} & 56.25\% &0.015 \\
    \textsc{base LM}&   & \textsc{joint BiT-Large*} & 58.93\% &0.00084 \\
    \textsc{base LM}&   & \textsc{BALM} & 46.77\% &0.88 \\
    \textsc{BALM}&   & \textsc{joint UniT} & 50.00\%& 0.52\\
    \textsc{BALM}&   & \textsc{joint BiT-Base} & 57.89\%& 0.0027\\
    \textsc{BALM}&   & \textsc{joint BiT-Large*} & 59.89\%& 0.00020\\
    \textsc{BALM-24L}&   & \textsc{joint BiT-Med} (24L) & 56.23\% & 0.015 \\
    \textsc{joint BiT-Large*} (24L)&   & \textsc{Human} &55.21\% & 0.036\\
    \midrule
    \textsc{base LM}& $\le$  & \textsc{BALM} & 54.85\% & 0.050\\
    \bottomrule\\
    \end{tabular}
    \caption{Human evaluation results on a subset of 333 sentences on the CC-News test set. The rate is computed as the percentage of sentences where the number of turkers preferring Model1 is strictly less than (denoted with $<$) or not greater than (denoted with $\le$) those preferring Model2. Attention check is used to drop some votes, so there might exist ties. p-value is based on single-sided binomial test. Note that for all models we used top-k sampling with $k=10$.}
    \label{tab:human}
    \vspace*{-0.2cm}
\end{table}

Better perplexity results do not necessarily imply better generations \citep{zhang2020trading,hashimoto2019unifying}: for example, approaches exhibiting mode-dropping problems such as language GANs might score terribly in terms of perplexity but do well in terms of generation \citep{yu2017seqgan,scialom2020coldgans}.
Besides, since generation from the residual EBM
requires approximations, the limited sample size might induce approximation errors
compared to truly sampling from the joint distribution. Therefore, we conducted human evaluations to compare generations
from the residual EBM model to generations from the baseline language models.

We generate from the joint model using Algorithm~\ref{algo:main} with $k=10$ and drawing
 10,000 samples from $\textsc{Base LM}$.  In open-domain generation, top-k/nucleaus sampling is necessary to get reasonable generations, as beam search or pure sampling result in degenerate sentences \citep{holtzman2019curious}. Therefore, both our baselines and our joint model's base language model used top-k sampling with $k=10$.
For each prefix, we present one completion from each model, and ask humans to select the one that is a better continuation. 
More details about human evaluation can be found in the Appendix~\ref{app:human}. 
The preference rates reported in Table~\ref{tab:human} confirm that indeed the generation quality of
 $\textsc{Joint Bit-Base}$ and $\textsc{Joint Bit-Large}*$ is better than both language model baselines. 
Depending on the model variant, our joint model (with bidirectional EBM) is preferred between 56\% and almost 60\% of the times;
interestingly, the preference rate does not change much as we compare against base LM as opposed to BALM.
In fact, humans do not seem to have a strong preference for BALM over base LM, despite the former scores two perplexity points lower.
Similarly, $\textsc{Joint Unit}$ is not strongly preferred over $\textsc{Base LM}$ despite its lower perplexity score.
We surmise that unidirectional scoring functions and auto-regressive models exhibit generation artifacts which are easily 
detected by humans, and these may overshadow the improvements brought by perplexity gains.  

\subsection{Automatic Assessment of Model Generations}
In the previous section we have demonstrated that the joint model produces better samples than the baseline auto-regressive language model according to humans. Can these generations better fool the discriminator trained to detect real from machine-generated text which we discussed in Section~\ref{sec:eval_general}? 

To answer this question and to provide an automatic way to assess generation quality, we have tested the false positive rate, that is the fraction of machine-generated samples that are deemed human-generated text, using as discriminator the BiT model. This is the classifier used in the last row of Table~\ref{tab:cross-corpus}, the one with best generalization accuracy since it was trained on all the corpora.
We found that the baseline language model $P_{\phi}$ (Base LM) has a false positive rate of 17.8\% while the joint language model $P_{\theta}$ (Joint BiT-med) has a much higher false positive rate of 31.8\%. 
The corresponding accuracy values are 89.9\% and 82.9\%.
Note that as we used only prefix of length 120 here, the numbers are not directly comparable with Table~\ref{tab:cross-corpus}.

In conclusion, samples from the joint language model are indeed harder to discriminate. This is expected since the residual EBM was precisely trained to detect machine-generated text and the resampling procedure used at generation time down-weights examples exhibiting machine generation artifacts and up-weights examples that are deemed most similar to genuine human generations (judged by the energy model). This experiment therefore demonstrates desirable generalization of the joint model: samples produced by the joint model are not just most similar to humans according to its own EBM classifier but also according to a discriminator trained with samples produced by an entirely different set of generators.  
 
\section{Analyses} \label{sec:analysis}
In this section, we analyze some of the results we obtained, both for the task of discriminating real from machine-generated text and for the text generation task.

\subsection{Effect of Prefix Length on Discrimination Accuracy}
First, we investigate the dependency between performance of the discriminators and length of the prefix.
We trained BiLSTMSmall and UniT models on examples with varying prefix length from the Wikitext corpus,
and computed the accuracy for each prefix length independently.
Figure~\ref{fig:accuracy_length} shows that as the prefix length increases (and the generation gets shorter),
the discrimination task gets harder and the difference between the models more prominent.
The unconditional case, i.e. zero prefix length, is the easiest, while prefixes of length 120 and 140 that are the main
experimental setup in this work, are the hardest.


\begin{figure}
\includegraphics[width=0.5\linewidth]{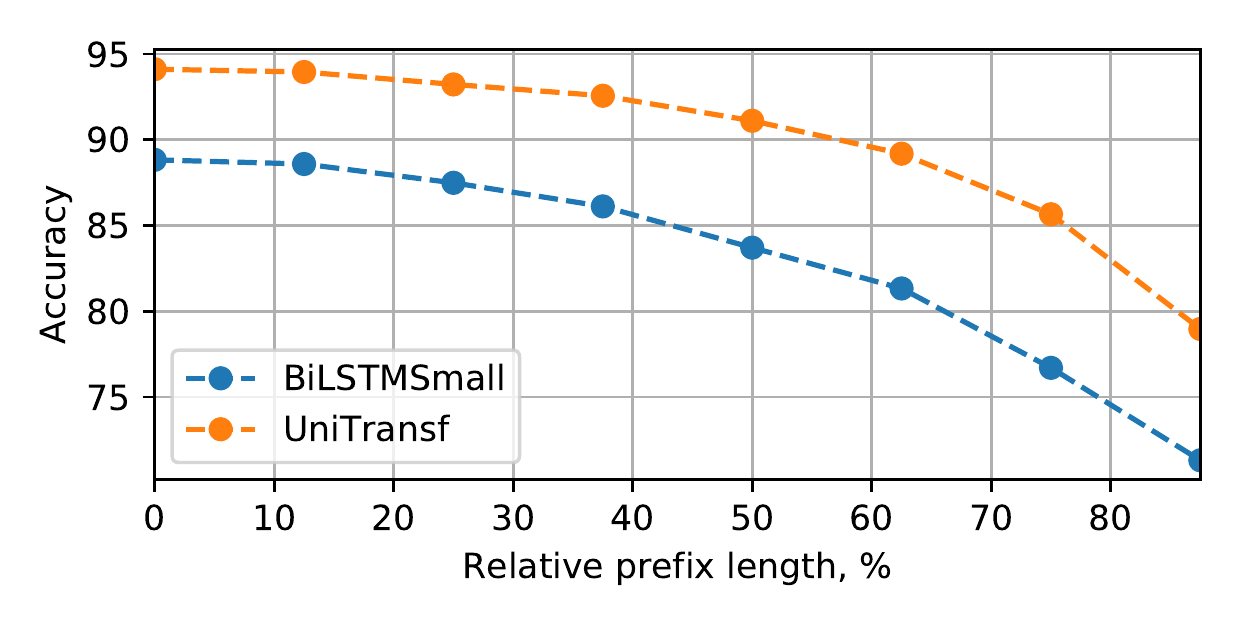}

\caption{\small Discrimination accuracy as a function of the ratio between the prefix length and the total length of the sequence
on the Wikitext dataset.}
\label{fig:accuracy_length}
\end{figure}

\subsection{Stability to Other Negative Distributions} \label{sec:stability}
\begin{figure}[!t]
\centering
\includegraphics[width=\linewidth]{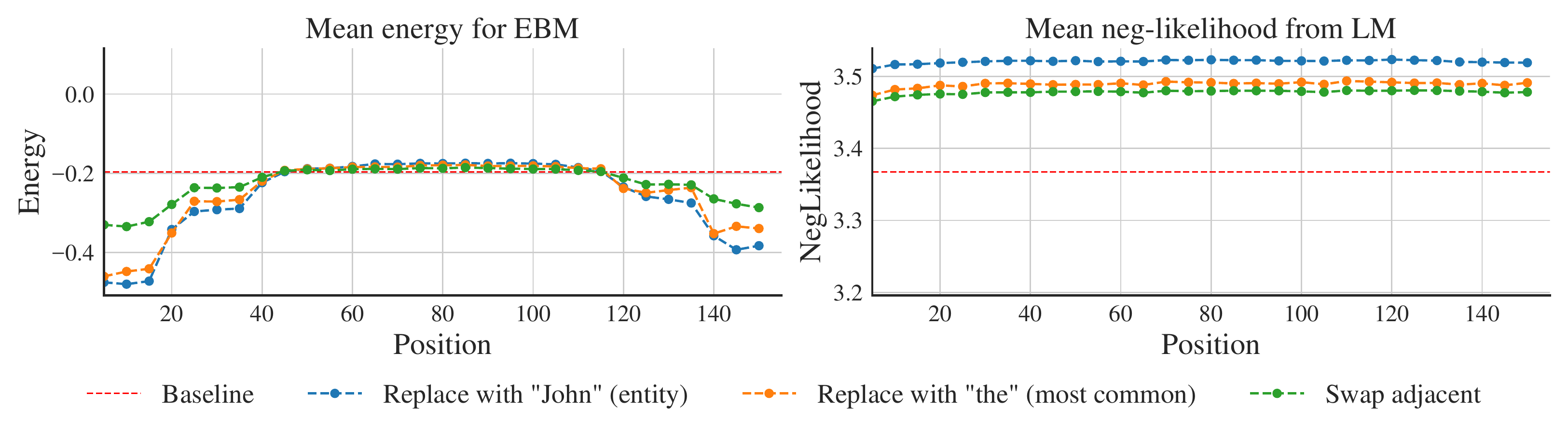}
\caption{\small
Effect of applying various perturbations (word replacement and swap of adjacent words) to ground-truth sequences
at different positions in terms of energy function and generator negative log-likelihood
(averaged over the whole test set of Wikitext).
The energy is only affected by corruptions at either end of the sequence.
These out-of-domain corruptions invariably decrease the energy.
  However, all perturbations increase the negative log-likelihood of the sequence.}
\label{fig:perturbation}
\end{figure}
In Section~\ref{sec:cross-corpus} we have seen that the energy function is less robust to negatives generated
from a model trained on a different corpus.
However, even in that case, a negative is still a sample from an auto-regressive neural network.
In Appendix~\ref{sec:exploring}, we show examples where changing a few entities can cause large
jumps in the energy (from negative to positive or vice versa), and so fool the EBM.
More generally, we see that the energy function is not robust to truly out-of-domain samples.
For example, the energy will score blocks of randomly generated text lower than real text.

These behaviors are evidence that the energy functions have learned the regularities of {\em generated} text,
as opposed to learning the regularities of real text. We surmise that it does so because modeling the latter would be much
more difficult than the former. By modeling generated text,
the energy function assigns low score to anything that is not generated by its training generator.

 While not surprising, this might be considered a liability of such energy functions.
However, as a model of text, the energy functions should be considered as working on the {\em residuals} of the
language models used to generate negatives.
   For the examples in Appendix~\ref{sec:exploring}, the language model records a large {\em decrease}
in likelihood after the change in entity; and language models of course give much lower likelihood to random text
than gold or generated text. Therefore, the energy function needs not to be accurate on examples that are already very unlikely according to these language models. These considerations further motivate our view of the EBM as a residual model and for optimizing for the joint distribution as specified in Eq.~\ref{eq:gen}.

In Figure~\ref{fig:perturbation} we show the average effects of applying various perturbations to sequences from
Wikitext103  on an in-domain energy and language model at each location (from 1 to 160) in the sequence.
We see that for all perturbations, the energy decreases its value, but the language model increases its negative log likelihood.
 We also see that the energy function is more sensitive to the ends of the text, which is where the negatives were different
from real text at training time.

\begin{figure}[t]
\begin{subfigure}{.5\textwidth}
  \centering
  \includegraphics[width=0.9\linewidth]{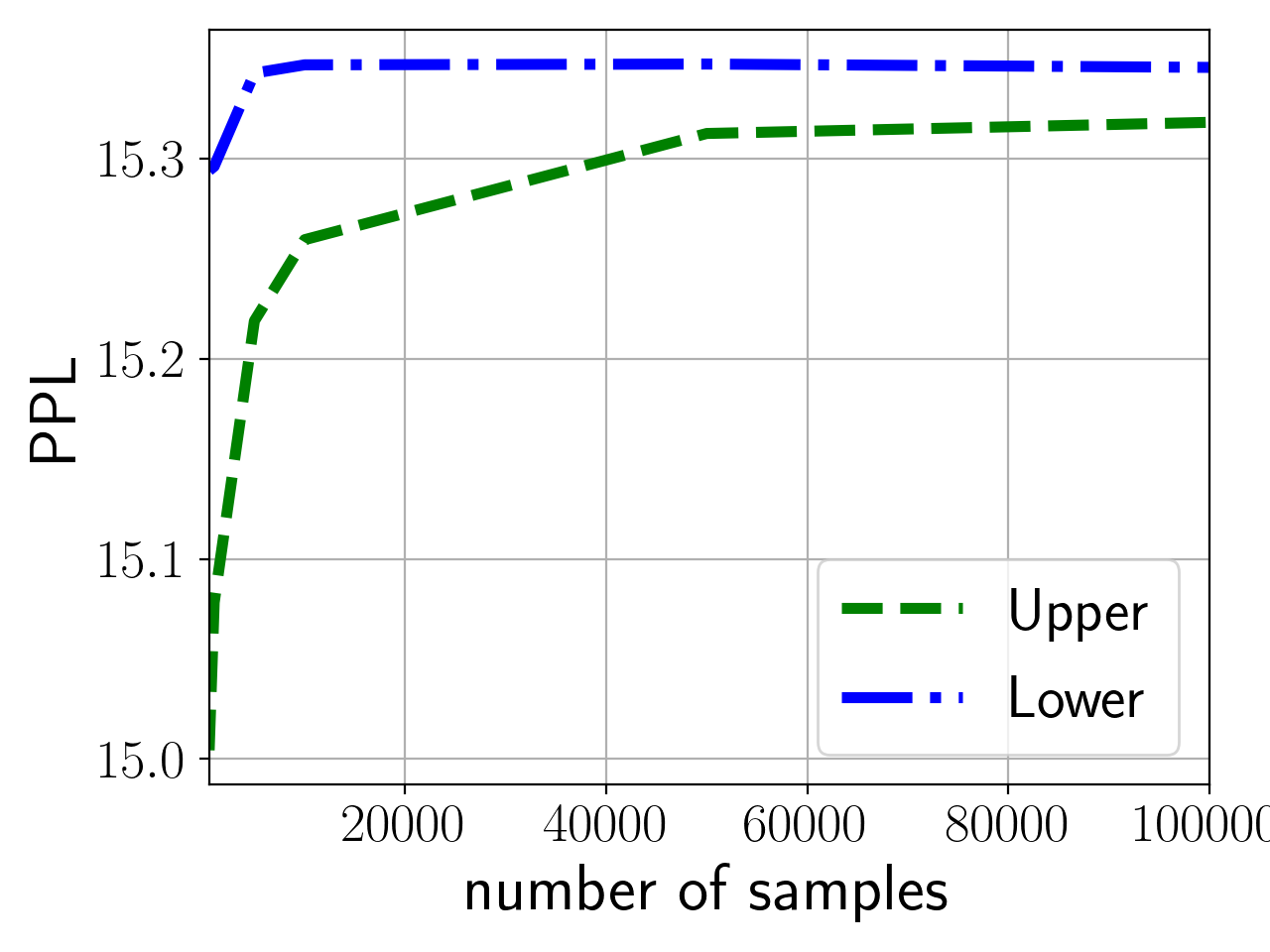}

\end{subfigure}%
\begin{subfigure}{.5\textwidth}
\includegraphics[width=0.9\linewidth]{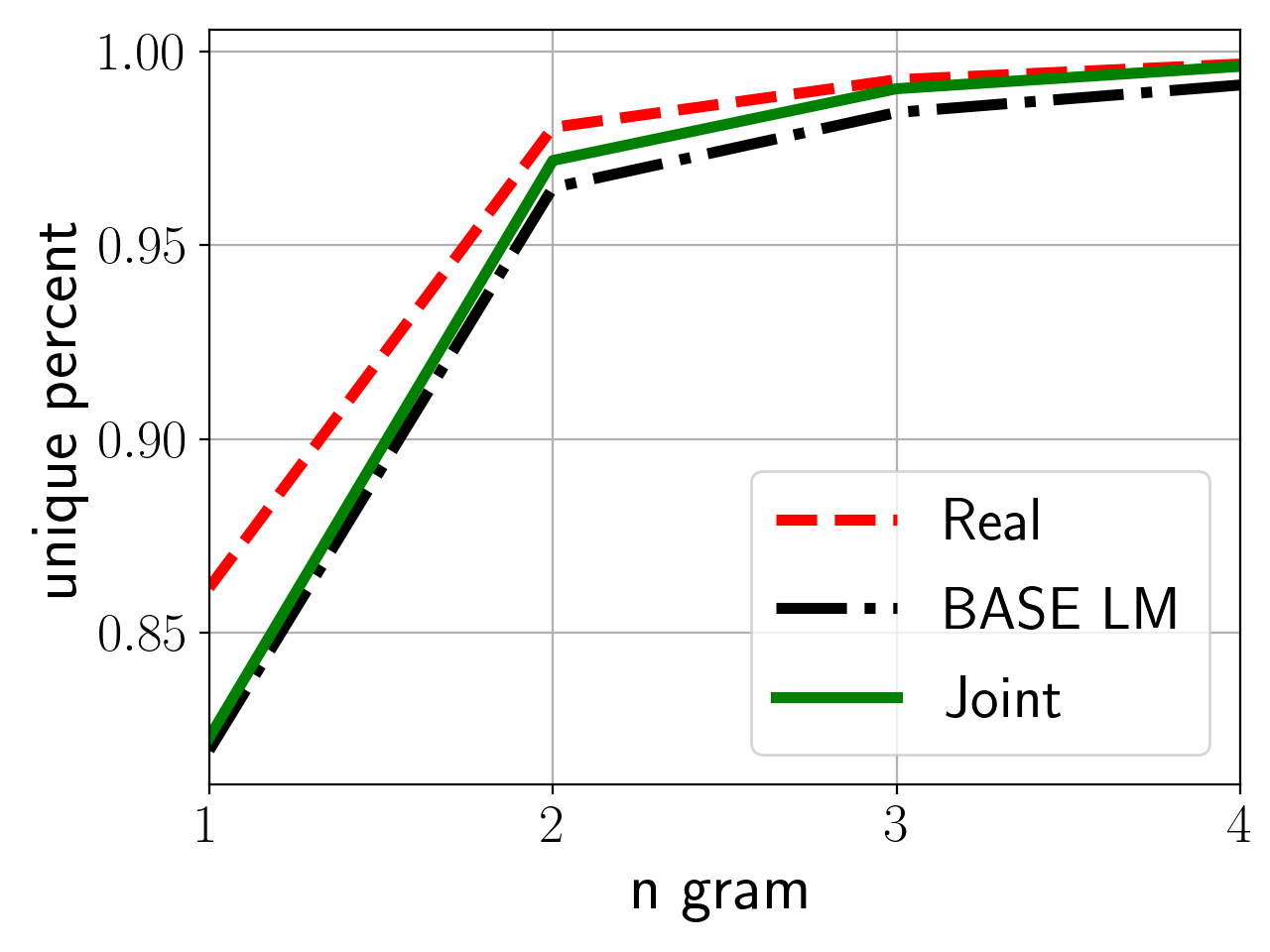}
\end{subfigure}
\caption{Left: PPL estimation for joint \textsc{BiT-Base} on CC-News validation set as we vary
the number of samples. Right: Percentage of Unique n-grams found in real data,
samples from the joint model $\textsc{BiT-Base}$ and samples from the base language model.
The joint sampling is done with 10,000 samples.}
\label{fig:repet}
\end{figure}

\subsection{Effect of Number of Samples in PPL estimates}
We now turn our attention to the use of the residual EBM for language modeling.
In Figure~\ref{fig:repet}, we vary the number of samples we take in order to
estimate PPL upper and lower bounds, see Section~\ref{sec:eval_ppl}.
Beyond 20,000 samples the upper estimate becomes very stable,
although we have to emphasize that these estimates might be biased even though the gap between lower and upper bound
closes as we take more samples.

\subsection{Analyzing Repetitions in Generations} \label{sec:repetitions}
A typical artifact of auto-regressive language models is their tendency to repeat phrases~\citep{holtzman2019curious, unlikely}.
It is then interesting to check whether the joint model is able to alleviate this artifact. Fig.~\ref{fig:repet}
shows that indeed the joint model has a slightly higher percentage of unique n-grams (unique per sample) compared to the baseline language model with $n=2,3,4$,
although still not as high as the original human-generated text. Appendix~\ref{appendix:examples} shows samples that got the highest energy score (hence very unlikely to sample during the resampling phase in Algorithm~\ref{algo:main}), most of which contain repetitions, which is a strong indicator of machine-generated text. This observation partly explains why we got slightly fewer repetitions in the joint model compared to the base language model. 

\subsection{A Necessary Condition for Matching the Data Distribution}
If the joint model $p_{\theta}$ matches the data distribution $p_d$, then statistics computed on a large population of samples
from the two distributions should also match. In particular, Fig.~\ref{fig:histog} show the density plots of log-likelihood scores
of the baseline language model (left) and joint model (right) when fed with their own samples versus samples from the test set.
We observe that the histogram of samples from the joint model matches the real data distribution more closely:
The difference of means in the $\textsc{LM Base}$ case is 21.64 whereas the difference is 6.20 in the joint approach.

\begin{figure}[t]
\begin{subfigure}{.5\textwidth}
  \centering
  \includegraphics[width=0.9\linewidth]{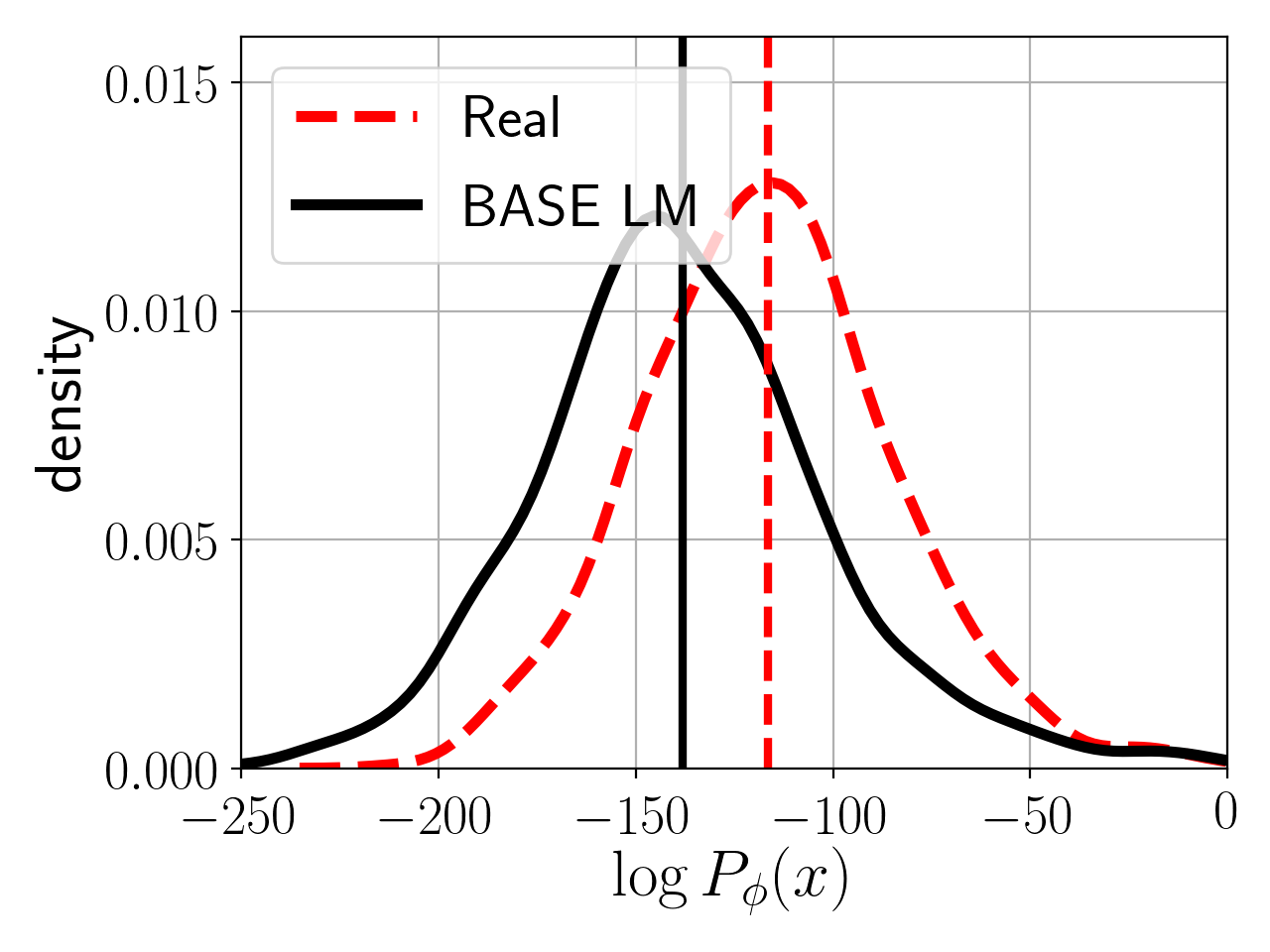}

\end{subfigure}%
\begin{subfigure}{.5\textwidth}
\includegraphics[width=0.9\linewidth]{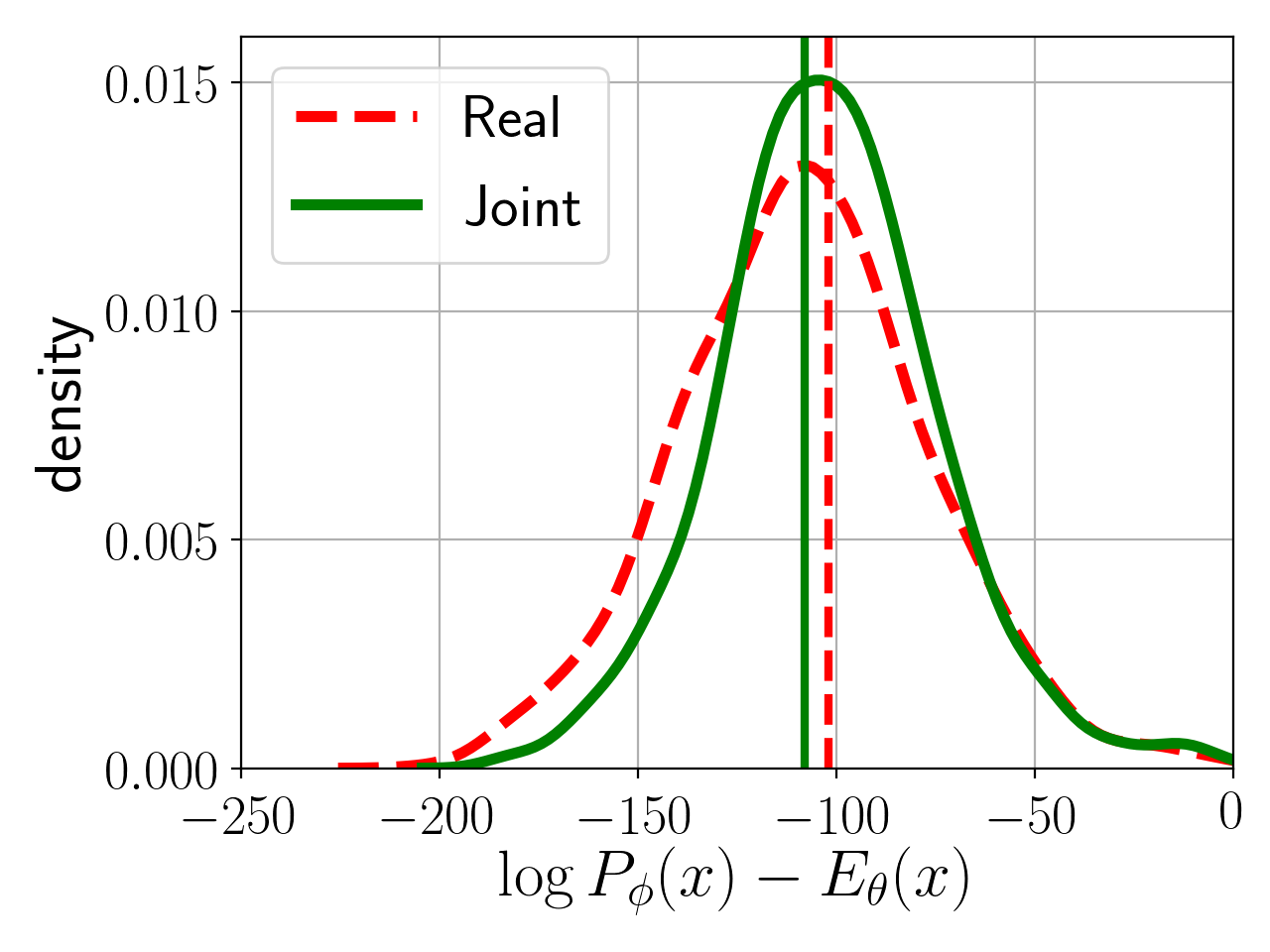}
\end{subfigure}
\caption{\label{fig:histog}Density plot of log-probability scores using the base
language model (left) or the joint model  (right). The red curve corresponds to real samples, the black curve to samples from \textsc{Base LM} and the green curve to samples from \textsc{BiT-Base}. The joint model provides a much better fit than the base language model.}
\end{figure}

\section{Limitations of Residual EBMs} \label{sec:limitations}
In the previous sections we highlighted the strengths of residual EBMs, namely their simplicity, efficiency both at training
and test time, their ability to generalize and their improved perplexity scores against strong auto-regressive language model baselines.
In this section, we comment on their limitations to caution the reader about when these methods are more likely to succeed
and to inform other  researchers about  what future avenues of research may naturally arise from this work.

In order to make training efficient and side-step computationally costly negative mining using the energy function itself, the current approach
uses negatives generated from a pretrained auto-regressive language model. Therefore, our model works as long as the base
language model from which we draw samples is strong enough, and as long as the ground truth and other plausible sequences are
{\em reachable} by the base language model.

If the base language model has poor quality, then generation from our
joint model is going to be poor as well, as the joint model merely resamples generations from the original language model. For example, in the extremity, if the base distribution is uniform, we might need to sample an exponential number of sentences to expect a sentence of good quality.
Moreover, training is going to be trivial if the base language model is poor, because the residual energy function
merely needs to detect trivial generation artifacts from the base language model. In fact, observe that the role of
positive and negative samples is symmetric in the loss of Eq.~\ref{eq:nce}. This means that the energy function can choose
to minimize the loss by either modeling the true data or the negative samples; since the latter have much simpler structure,
it is going to model the negative samples. Therefore, importance sampling amounts to mostly down-weighting the worst samples
from the base language model, as already discussed in Section~\ref{sec:repetitions}.
The consequence of this is that search with a poor base language model is going to be catastrophically inefficient, as
we would need to sample an impractically large number of negatives in order to find samples that are reasonably close
to the true data manifold.

To summarize, this work makes a rather strong implicit assumption on the quality of the base language model, and it
is expected to work well only when this is rather strong. In our application, this assumption is met quite well in practice
as large auto-regressive language models trained on large datasets have improved significantly in recent years~\citep{radford2019language}.
In general however, residual learning always carries liability to its base model.


\section{Conclusions}
The EBM framework could potentially unlock more expressive models of text, as they are not limited to scoring a
single word at a time as current locally normalized auto-regressive models do.
Unfortunately, training EBMs is challenging because generating negatives using the energy function itself is still an open research problem,
and does not scale well in practice. In this work, we leverage generations produced by
pretrained language models as negative samples~\citep{wang2018learning,parshakova2019global}.

As a preliminary yet necessary step in this direction we have investigated the generalization ability of such EBMs.
We found that EBMs, when trained on large datasets, achieve  good generalization. For instance, they
behave nicely when tested with negatives produced by generators that have rather different architectures.
The generalization is less good when generators are trained on other corpora, but EBMs re-gain robustness once we train
them on even bigger composite datasets.

Finally, we showed that such EBMs can be used to improve text generation. Generation is efficient as it amounts to resampling
from the large set of negatives produced by the base language model. Our estimates show that the resulting model has lower perplexity than the base
language model, and our human evaluation confirms that generations from the joint model are preferred to generations from the base language model.
Finally, this approach may be interpreted
as a natural way to finetune a large pretrained bidirectional Transformer like BERT for text generation applications.

\section{Future Work}
The goal of this work is to improve upon strong language
        models on large scale datasets. The field has also been
        focusing on this setting, as shown by recent
        works~\citep{radford2019language,brown2020language}. However,
        the use of weaker base language models brings interesting
        theoretical questions. For example, in the limit case when the
        baseline distribution is uniform, we might need to sample an
        exponential number of sentences to expect a sentence of good
        quality. We believe this is an interesting direction of future
        research.

        In the future, we can also improve EBMs for text by simply making their architectures bigger and
increasing the diversity and size of their training datasets. Of course, further scaling up of EBMs will pose
formidable engineering challenges.

On the application side, a natural application of the current formulation of EBMs is real/fake text discrimination.
We believe that  this is important in its own right, and that EBMs can be powerful, as demonstrated
by their superior performance compared to discriminating using the original language model log-likelihood.
We also plan to investigate other ways to generate negatives that may strike a better trade-off
between the amount of compute each negative requires and their closeness to the joint model distribution.
It would also be interesting to explore other loss functions and the generation of longer pieces of text by using this
model auto-regressively at the chunk level, as opposed to the token level.

In this work we only studied the task of language modeling. Language
models are the fundamental building block of modern neural text
generation models. Our promising results on language modeling point to
future directions in applying the residual EBM approach to other text
generation tasks, such as speech recognition and machine
translation~\citep{naskar2020energy}. With the additional source side
information constraint, it might be harder for the model to satisfy
the self-normalizing property under various inputs \citep{ma18,
  roark2007discriminative}. At the same time, this constraint might
make the partition function estimation easier. These questions need to be answered empirically, and we leave them for future work.



\newpage

\appendix

\section{Hyper-parameter Setting} \label{app:hyperparams}
All models are implemented using the PyTorch framework~\citep{paszke2017automatic} and are optimized using Adam~\citep{kingma2014adam}.
To train our biggest models (UniT and BiT) we used several machines each with 8 GPUs in synchronous mode using data parallelism.
The resulting large batch size speeds up training when combined with float16 reduced precision and cosine scheduling of the learning rate
without any restarts~\citep{loshchilov2016sgdr}, i.e. we decay the learning rate to zero over the course of ``max steps'' updates and then stop training.
Using these methods, we reduced training time by five times compared to a single node training.
For simpler models we used a single node with up to 8 GPUs and inverse square root decay.

\begin{table}[!h]
\center
    \small
\begin{tabular}{llcccccc}
\toprule
Model & max lr & bsz (/GPU) & GPUs & fp16 & warmup steps & max steps & max grad norm \\
\midrule
Linear              & 0.01  & 1024 & 1   & + & 1000 & -           &          0.25 \\
BiLSTM              & 0.0002 & 128 & 8   & + & 1000 & -           &          0.25 \\
UniT           & 0.0003 & 32  & 64  & + & 2000 & 180000      &          0.25 \\
BiT            & 0.00005& 20  & 192 & + & 2000 & 180000      &          10.0 \\
\bottomrule
\end{tabular}
\caption{\small Hyper-parameter values used in our real/fake discrimination experiments.}
\label{tbl:hyperparams}
\end{table}

\begin{table}[!hp]
\centering
\scriptsize
\begin{tabular}{@{}llcccccc@{}}
\toprule
Model                                           & max lr  & bsz (/GPU) & GPUs & fp16 & warmup steps & max steps & max grad norm \\ 
\midrule
\textsc{base LM}                   & 0.0001  & 32            & 64   & -    & 2,000        & 180,000   & 10            \\
\textsc{RALM}                      & 0.0001  & 64            & 64   & -    & 2,000        & 180,000   & 10            \\
\textsc{BALM}                      & 0.0001  & 32            & 64   & -    & 2,000        & 180,000   & 10            \\
\textsc{joint UniT}                & 0.0003  & 64            & 64   & +    & 2,000        & 180,000   & 10            \\
\textsc{joint BiT-Base}            & 0.00005 & 60            & 64   & -    & 2,000        & 90,000    & 0.25          \\
\textsc{joint BiT-Base*}           & 0.00005 & 60            & 64   & -    & 2,000        & 90,000    & 0.25          \\
\textsc{joint BiT-Large*}          & 0.0003  & 64            & 64   & +    & 2,000        & 90,000    & 10            \\ 
\midrule
\textsc{base LM-24L}               & 0.0003  & 50            & 64   & -    & 2,000        & 90,000    & 0.25          \\
\textsc{RALM-24L}                  & 0.00015 & 28            & 64   & -    & 1,000        & 90,000    & 0.25          \\
\textsc{BALM-24L}                  & 0.0003  & 28            & 64   & -    & 2,000        & 90,000    & 0.25          \\
\textsc{joint UniT} (24L)       & 0.0003  & 64            & 64   & +    & 2,000        & 180,000   & 10            \\
\textsc{joint BiT-Base} (24L)   & 0.00005 & 60            & 64   & -    & 2,000        & 90,000    & 0.25          \\
\textsc{joint BiT-Base*} (24L)  & 0.00005 & 60            & 64   & -    & 2,000        & 90,000    & 0.25          \\
\textsc{joint BiT-Med} (24L)    & 0.00005 & 32            & 64   & -    & 2,000        & 90,000    & 0.25          \\
\textsc{joint BiT-Large*} (24L) & 0.00005 & 20            & 64   & -    & 2,000        & 90,000    & 0.25      \\
\bottomrule
\end{tabular}
\caption{\small Hyper-parameter values used in our language modeling and text generation experiments.}
\end{table}

\newpage

\section{\label{sec:exploring}Perturbing the Energy Function}
In this section we show that we can change a few words to make a negative example become a ``positive'' one
as judged by the energy function (alone), and vice versa, by using gradient information.

Below here, we show an example of a ground truth sentence from the Wikitext dataset.

\begin{framed}
 <EOS> =Robert Boulter= <EOS>  <EOS> Robert Boulter is an English film, television and theatre actor. He had a guest-starring role on the television series The Bill in 2000. This was followed by a starring role in the play Herons written by Simon Stephens, which was performed in 2001 at the Royal Court Theatre. He had a guest role in the television series Judge John Deed in 2002. In 2004 Boulter landed a role as "Craig" in the episode "Teddy's Story" of the television series The Long Firm; he starred alongside actors Mark Strong and\generategap{ Derek Jacobi. He was cast in the 2005 theatre productions of the Philip Ridley play Mercury Fur, which was performed at the Drum Theatre in Plymouth and the Menier Chocolate Factory in London. He was}
\end{framed}

Here the block has 160 BPE tokens, where the first 120 tokens (black font) are used as context and the remaining
40 are the ground truth completion.
Next, we use a language model to generate 10 negatives:

\begin{framed}
\textit{Negative 1.}\enskip <EOS> =Robert Boulter= <EOS>  <EOS> Robert Boulter is an English film, television and theatre actor. He had a guest-starring role on the television series The Bill in 2000. This was followed by a starring role in the play Herons written by Simon Stephens, which was performed in 2001 at the Royal Court Theatre. He had a guest role in the television series Judge John Deed in 2002. In 2004 Boulter landed a role as "Craig" in the episode "Teddy's Story" of the television series The Long Firm; he starred alongside actors Mark Strong and\generategap{ Chris Elliott in 2006 as the character. Boulter has appeared in various television specials dealing with the series since its inception. <EOS> After graduating with a degree in drama, Boulter worked as a}

\textit{Negative 2.}\enskip <EOS> =Robert Boulter= <EOS>  <EOS> Robert Boulter is an English film, television and theatre actor. He had a guest-starring role on the television series The Bill in 2000. This was followed by a starring role in the play Herons written by Simon Stephens, which was performed in 2001 at the Royal Court Theatre. He had a guest role in the television series Judge John Deed in 2002. In 2004 Boulter landed a role as "Craig" in the episode "Teddy's Story" of the television series The Long Firm; he starred alongside actors Mark Strong and\generategap{ Stephen Fry in the episode "You're All Alone" and in the episode "The Longest Day". <EOS> He auditioned for the role in the series in 2003 but was not cast. In 2005}

$\vdots$

\textit{Negative 10.}\enskip <EOS> =Robert Boulter= <EOS>  <EOS> Robert Boulter is an English film, television and theatre actor. He had a guest-starring role on the television series The Bill in 2000. This was followed by a starring role in the play Herons written by Simon Stephens, which was performed in 2001 at the Royal Court Theatre. He had a guest role in the television series Judge John Deed in 2002. In 2004 Boulter landed a role as "Craig" in the episode "Teddy's Story" of the television series The Long Firm; he starred alongside actors Mark Strong and\generategap{ Ian Somerhalder on the BBC series Top Gear; this was followed up in 2007 by a role in the BBC science-fiction series Doctor Who. In 2008 Boulter appeared in the BBC}
\end{framed}


In this example, using the big Transformer model, UniT, as the energy function,
we are able to separate real from fake examples
as shown. We want to perturb these negatives to violate the margin.
To do so, we make use of the gradient information from the energy function $\nabla_{x} E_\theta (x)$
and use a first order Taylor expansion to
approximate the effect of a token replacement (we abuse our notations and use $x$ to denote embeddings in this analysis).
Given the original sample $x$, we change one word $x_i$ to $x_i'$ to arrive at $x'$. The score of $x'$ is approximately:

\begin{equation}
E_\theta(x) + \nabla_{x_i} E_\theta(x) \cdot (x_i'-x_i)
\end{equation}

Using this approximation, we can search for those token replacements that increase/decrease the energy the most.
We can easily change a negative sample to a positive one by replacing the 5 words highlighted below.
 In paratheses, we report both score and language model perplexity.

\begin{framed}
\textit{Original negative (score -0.77, PPL 20.77).}\enskip  <EOS> =Robert Boulter= <EOS>  <EOS> Robert Boulter is an English film, television and theatre actor. He had a guest-starring role on the television series The Bill in 2000. This was followed by a starring role in the play Herons written by Simon Stephens, which was performed in 2001 at the Royal Court Theatre. He had a guest role in the television series Judge John Deed in 2002. In 2004 Boulter landed a role as "Craig" in the episode "Teddy's Story" of the television series The Long Firm; he starred alongside actors Mark Strong and\markgap{ Chris}\markgap{ Elliott} in 2006 as the character. Boulter has appeared in various television specials\markgap{ dealing} with the series since its inception. <EOS> After graduating with a degree in\markgap{ drama}, Boulter worked as a

\textit{Perturbed negative (score 0.00, PPL 117.30).}\enskip <EOS> =Robert Boulter= <EOS>  <EOS> Robert Boulter is an English film, television and theatre actor. He had a guest-starring role on the television series The Bill in 2000. This was followed by a starring role in the play Herons written by Simon Stephens, which was performed in 2001 at the Royal Court Theatre. He had a guest role in the television series Judge John Deed in 2002. In 2004 Boulter landed a role as "Craig" in the episode "Teddy's Story" of the television series The Long Firm; he starred alongside actors Mark Strong and\markgapsingle{ Gor}{-0.0.64}{28.97}\markgapsingle{ Trem}{-0.56}{38.86} in 2006 as the character. Boulter has appeared in various television specials\markgapsingle{ relates}{-0.77}{24.60} with the series since its inception. <EOS> After\markgapsingle{Health}{-0.35}{39.52} with a degree in\markgapsingle{edited}{-0.49}{27.45}, Boulter worked as a
\end{framed}

In the above example, we also show the (score, PPL) for replacing a single token in the subscripts. Similarly, we can replace a few words and make a positive sample become negative.

\begin{framed}
\textit{Original positive (score -0.25, PPL 77.68).}\enskip  <EOS> =Robert Boulter= <EOS>  <EOS> Robert Boulter is an English film, television and theatre actor. He had a guest-starring role on the television series The Bill in 2000. This was followed by a starring role in the play Herons written by Simon Stephens, which was performed in 2001 at the Royal Court Theatre. He had a guest role in the television series Judge John Deed in 2002. In 2004 Boulter landed a role as "Craig" in the episode "Teddy's Story" of the television series The Long Firm; he starred alongside actors Mark Strong and\markgap{ Derek} Jacobi. He was cast in the 2005 theatre productions of the Philip Ridley play Mercury Fur, which was performed at the\markgap{ Drum} Theatre in\markgap{ Plymouth} and the\markgap{ Men}ier\markgap{ Chocolate} Factory in London. He was

\textit{Perturbed positive (score -0.78, PPL 142.85).}\enskip  <EOS> =Robert Boulter= <EOS>  <EOS> Robert Boulter is an English film, television and theatre actor. He had a guest-starring role on the television series The Bill in 2000. This was followed by a starring role in the play Herons written by Simon Stephens, which was performed in 2001 at the Royal Court Theatre. He had a guest role in the television series Judge John Deed in 2002. In 2004 Boulter landed a role as "Craig" in the episode "Teddy's Story" of the television series The Long Firm; he starred alongside actors Mark Strong and\markgapsingle{connected}{-0.30}{118.30} Jacobi. He was cast in the 2005 theatre productions of the Philip Ridley play Mercury Fur, which was performed at the\markgapsingle{ C}{-0.28}{75.36} Theatre in\markgapsingle{ London}{-0.47}{62.29} and the\markgapsingle{ Vaughan}{-0.40}{93.77}ier\markgapsingle{cerning}{-0.32}{100.71} Factory in London. He was
\end{framed}

\begin{figure}[t]
\centering
\includegraphics[width=.5\linewidth]{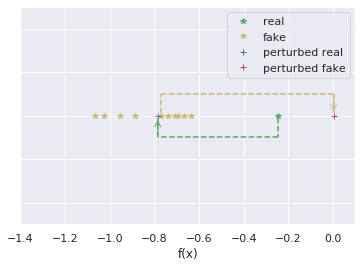}
\caption{\small By changing a few words we can make a negative sample become real as scored by the (negative)
ennergy function, and vice versa.}
\label{fig:margin_edited}
\end{figure}

As shown in Figure~\ref{fig:margin_edited}, we can easily ``fool'' the discriminator by editing a few words.
However, these edited sentences have a very low probability (high PPL) under the generator we used.
This explains why the discriminator gets fooled, because it has never seen such negatives during training.

\newpage
\section{\label{appendix:qual}Top-k auto-regressive sampling}
In this subsection, we factorize the joint model \textsc{BiT-Base} auto-regressively, and compare its differences with \textsc{Base LM}. Since even estimating the per step probabilities according to Eq.~\ref{eq:stepppl} is too computationally expensive, we further approximate it by only considering the top 128 words predicted by \textsc{Base LM}, where we sample 10,000 completions for each of them to estimate $P(x_t|x_{<t})$. Then we take the top 10 entries and re-normalize, and compare it to the top 10 probabilities of \textsc{Base LM}. Note that this procedure cannot be used as a properly normalized model to evaluate perplexity, because these top 128 words might not contain the ground truth.

Our initial explorations suggested that the joint model tends to generate fewer repetitions. Therefore we picked a few LM samples where there are repetitions at $x_t$, and use the same context $x_{<t}$ to estimate $P(x_t|x_{<t})$ for the joint model. Some examples of $P(x_t|x_{<t})$ of \textsc{Base LM} and \textsc{BiT-Base} are presented in Table~\ref{tab:qual}. Indeed \textsc{Base LM} usually assigns lower probabilities to repetitions even though the top k words remain the same, which is not surprising given that the existence of repetition is a strong indicator of coming from the LM, which would lead to a higher energy value hence lower joint probability.

\begin{table}[!htp]
  \tiny
    \centering
    \begin{tabular}{p{6cm}cccc}
    \toprule
    Context $x_{<t}$ & Model & Rank     &  $x_t$ & $P(x_t|x_{<t})$  \\
    \midrule
     \multirow{10}{6cm}{\footnote{Excerpt from \url{ https://www.swissinfo.ch/eng/multinational-principles_swiss-government-gives-green-light-for-un-migration-accord/44464186}.}... is aimed at setting common benchmarks for orderly migration practices, thereby reducing irregular flows. The Global Compact contains ten guiding principles, including that migrants cannot be settled by countries with better integration policies and a fair and sustainable development. "For the first time in our history, a {\color{red}legally binding} and } &\multirow{5}{*}{\textsc{Base LM}} & 0  & {\color{red}binding}&  {\color{red}0.39}\\
     & & 1  & {\color{red}legally}&  {\color{red}0.33}\\
     & & 2  & internationally&  0.06\\
     & & 3  & comprehensive&  0.05\\
     & & 4  & {transparent}&  0.04\\
     \cline{2-5}
     &\multirow{5}{*}{\textsc{BiT-Base}} & 0  & {\color{red}binding}&  {\color{red}0.18}\\
     & & 1  & {\color{red}legally}&  {\color{red}0.17}\\
     & & 2  & internationally&  0.12\\
     & & 3  & comprehensive&  0.09\\
     & & 4  & {transparent}&  0.08\\
     \midrule
         \multirow{10}{6cm}{
\footnote{Excerpt from \url{https://www.forbes.com/sites/markmurphy/2018/05/11/this-is-the-one-piece-of-data-that-85-of-recruiters-are-missing/\#25917c765dad}.}
... companies that land their first-choice candidates 90-100\% of the time, 24\% of them have "thoroughly defined" their high performer attitudes. By contrast, only 1\% of companies that struggle to land their first-choice candidates "thoroughly defined" their high performer attitudes. So it seems pretty clear that companies that land their top-choice candidates are not always as {\color{red}willing} and }
&\multirow{5}{*}{\textsc{Base LM}} & 0  & {able}&  {0.66}\\
     & & 1  & {\color{red}willing}&  {\color{red}0.09}\\
     & & 2  & eager&  0.07\\
     & & 3  & ready &  0.05\\
     & & 4  & {well}&  0.04\\
     \cline{2-5}
     &\multirow{5}{*}{\textsc{BiT-Base}} & 0  & {able}&  {0.75}\\
     & & 1  & {\color{red}willing}&  {\color{red}0.05   }\\
     & & 2  & eager&  0.05\\
     & & 3  & ready&  0.04\\
     & & 4  & {well}&  0.03\\
     \midrule
     \multirow{10}{6cm}{\footnote{Excerpt from \url{https://www.reuters.com/article/us-usa-fed-powell/fed-nominee-powell-once-hawkish-now-champions-yellens-focus-on-jobs-idUSKBN1DS0FG}}... it reveals a key skill needed to lead the Fed. "You need to know what you don't know. And you need to be willing to listen when you don't know something," said Karen Dynan, who as an assistant Treasury Secretary in Barack Obama's second administration would regularly meet Fed governors. <EOS> New Delhi Dec 5 The following are mergers under review by India's {\color{red}financial} services and  } &\multirow{5}{*}{\textsc{Base LM}} & 0  & {banking}&  {0.64}\\
     & & 1  & {\color{red}financial}&  {\color{red}0.10}\\
     & & 2  & insurance&  0.09\\
     & & 3  & technology&  0.05\\
     & & 4  & {IT}&  0.04\\
     \cline{2-5}
     &\multirow{5}{*}{\textsc{BiT-Base}} & 0  & {banking}&  {0.92}\\
     & & 1  & {\color{red}financial}&  {\color{red}0.06}\\
     & & 2  & insurance&  0.01\\
     & & 3  & technology&  0.00\\
     & & 4  & {IT}&  0.00\\
         \bottomrule\\
    \end{tabular}
    \caption{Comparison of $P(x_t|x_{<t})$ between \textsc{Base LM} and \textsc{BiT-Base} on a few examples. Repetitions are marked with red. Only the top 5 probabilities are shown.}
    \label{tab:qual}
\end{table}

\newpage
\section{Proof of Theorem 2} \label{appendix:proof}
\setcounter{theorem}{1}
\begin{theorem}
Denote $T_n$ as the empirical estimate of $\log \mathbb{E}_{x\sim P_{\phi}} \exp(-E(x))$ with $n$ samples $x^i\sim P_{\phi}$ $(i=1,\cdots, n)$, i.e., $T_n = \log \frac{1}{n}\sum_{i=1}^n\exp(-E(x^i)) $, then $\forall \epsilon > 0$, $\exists N>0$ such that $\forall n > N$ we have
\begin{equation}
      Z_\theta -\epsilon < \mathbb{E}[T_n] < Z_\theta  < \mathbb{E}[(2n-1) T_n - 2(n-1) T_{n-1}] < Z_\theta +\epsilon
\end{equation}
\end{theorem}
\begin{proof}
From \citet{nowozin2018debiasing} Eq. 35, we can write $\mathbb{E}[T_n]$ as
\begin{multline}
    \mathbb{E}[T_n] = Z_\theta - \frac{\mu_2}{2 \mu^2}\frac{1}{n} + \frac{1}{3\mu^3}\frac{\mu_3}{n^2} - \frac{1}{4\mu^4}(\frac{3}{n^2}\mu_2^2 + \frac{1}{n^3}(\mu_4 - 3\mu_2^2)) \\+ \frac{1}{5\mu^5} (\frac{10}{n^3}\mu_3\mu_2 +\frac{1}{n^4}(\mu_5 - 10 \mu_3\mu_2)) + o(n^{-3})
    \label{eq:taylor}
\end{multline}
Where $\mu = \mathbb{E}[T_n]$, $\mu_k = \mathbb{E}[(T_n-\mu)^k]$. Equivalently,
\begin{equation}
    \mathbb{E}[T_n] = Z_\theta - \frac{\mu_2}{2 \mu^2}\frac{1}{n} + o(n^{-1})
\end{equation}
Therefore, $\lim_{n\to\infty} \mathbb{E}[T_n] = Z_\theta$. So $\forall \epsilon > 0$, $\exists N_1>0$ such that when $n> N_1$, $\mathbb{E}[T_n] > Z_\theta - \epsilon$. On the other hand, $\lim_{n\to \infty}n(Z_\theta - \mathbb{E}[T_n]) = \lim_{n\to \infty}\frac{\mu_2}{2 \mu^2} + o(1) = \frac{\mu_2}{2 \mu^2} > 0$ (because $\lim_{n\to\infty} o(1) = 0$), so $\exists N_2>0$ such that when $n>N_2$ we have $Z_\theta > \mathbb{E}[T_n]$. Up to this point, we have proved that $Z_\theta -\epsilon < \mathbb{E}[T_n] < Z_\theta$.

For the other half part of the proof, using Eq.~\ref{eq:taylor} we have
\begin{equation}\mathbb{E}[T_n] = Z_\theta - \frac{\mu_2}{2\mu^2}\frac{1}{n} + \frac{c}{n^2} + o(n^{-2})
\end{equation}
where $c$ is a constant. Therefore, $\mathbb{E}[(2n-1) T_n - 2(n-1) T_{n-1}] = (2n-1) \mathbb{E}[T_n] - 2(n-1)\mathbb{E}[T_{n-1}] = Z_\theta + \frac{\mu_2}{2\mu^2}\frac{1}{n} + o(n^{-1})$. Therefore $\lim_{n\to\infty} \mathbb{E}[(2n-1) T_n - 2(n-1) T_{n-1}] = Z_\theta$, hence $\forall \epsilon>0$, $\exists N_3>0$ such that $\forall n>N_3$ $\mathbb{E}[(2n-1) T_n - 2(n-1) T_{n-1}] < Z_\theta + \epsilon$. Furthermore, $\lim_{n\to\infty}n(\mathbb{E}[(2n-1) T_n - 2(n-1) T_{n-1}] - Z_\theta) = \lim_{n\to\infty} \frac{\mu_2}{2\mu^2} + o(1) > 0$, so $\exists N_4>0$ such that when $n>N_4$ we have $\mathbb{E}[(2n-1) T_n - 2(n-1) T_{n-1}>Z_\theta$.

Putting the above together, $\forall \epsilon>0$, let $N=\max\{N_1,N_2,N_3,N_4\}$, then $\forall n>N$,
\begin{equation}
      Z_\theta -\epsilon < \mathbb{E}[T_n] < Z_\theta  < \mathbb{E}[(2n-1) T_n - 2(n-1) T_{n-1}] < Z_\theta +\epsilon
\end{equation}
\end{proof}

\section{Derivation of Eq.~\ref{eq:stepppl}} \label{app:proof6}
Without loss of generality, we ignore the prefix $x_1, \cdots, x_p$.
\begin{align*}
   & P(x_1, \cdots, x_{t}) = \sum_{x_{t+1}, \cdots, x_T} P(x_1, \cdots, x_T)\\
   & = \sum_{x_{t+1}, \cdots, x_T} P_{\phi}(x_1, \cdots, x_T)\exp (-E_\theta(x_{1},\cdots, x_T)) / Z_{\theta}\\
   & = \sum_{x_{t+1}, \cdots, x_T} P_{\phi}(x_1, \cdots, x_t) P_{\phi}(x_{t+1}, \cdots, x_T|x_1, \cdots, x_t)\exp (-E_\theta(x_{1}, \cdots, x_T)) / Z_{\theta}\\
   & = \mathbb{E}_{x_{t+1}, \cdots, x_T \sim P_{\phi}(x_{t+1}, \cdots, x_T|x_1, \cdots, x_t)} P_{\phi}(x_1, \cdots, x_t) \exp (-E_\theta(x_{1}, \cdots, x_T)) / Z_{\theta}\\
   & = P_{\phi}(x_1, \cdots, x_t) \mathbb{E}_{x_{t+1}, \cdots, x_T \sim P_{\phi}(x_{t+1}, \cdots, x_T|x_1, \cdots, x_t)}  \exp (-E_\theta(x_{1}, \cdots, x_T)) / Z_{\theta}\\
   & = P_{\phi}(x_1, \cdots, x_t) \mathbb{E}_{x_{t+1}', \cdots, x_T' \sim P_{\phi}(\cdot|x_1, \cdots, x_t)}  \exp (-E_\theta(x_{1}, \cdots,x_t, x_{t+1}',\cdots, x_T')) / Z_{\theta}\\
\end{align*}
Therefore,
\begin{align*}
    &P(x_t | x_{<t}) = \frac{P(x_{1}, \cdots, x_t)}{P(x_1, \cdots, x_{t-1})}\\
    & = \frac{P_{\phi}(x_1, \cdots, x_t) \mathbb{E}_{x_{t+1}', \cdots, x_T' \sim P_{\phi}(\cdot|x_1, \cdots, x_t)}  \exp (-E_\theta(x_{1}, \cdots,x_t, x_{t+1}',\cdots, x_T')) / Z_{\theta}}{P_{\phi}(x_1, \cdots, x_{t-1}) \mathbb{E}_{x_{t}', \cdots, x_T' \sim P_{\phi}(\cdot|x_1, \cdots, x_{t-1})}  \exp (-E_\theta(x_{1}, \cdots,x_{t-1}, x_t',\cdots, x_T')) / Z_{\theta}}\\
    & = \frac{P_{\phi}(x_1, \cdots, x_t) \mathbb{E}_{x_{t+1}', \cdots, x_T' \sim P_{\phi}(\cdot|x_{\le t})}  \exp (-E_\theta(x_{\le t}, x_{t+1}',\cdots, x_T')) }{P_{\phi}(x_1, \cdots, x_{t-1}) \mathbb{E}_{x_{t}', \cdots, x_T' \sim P_{\phi}(\cdot|x_{\le t-1})}  \exp (-E_\theta(x_{\le t-1}, x_t',\cdots, x_T')) }\\
    &=P_{\phi}(x_t|x_{<t}) \frac{\mathbb{E}_{x_{t+1}',\cdots, x_{T}' \sim P_{\phi}(\cdot | x_{\le t})}[\exp (-E_\theta(x_{\le t}, x_{t+1}',\cdots, x_T'))]}{\mathbb{E}_{x_{t}',\cdots, x_{T}' \sim P_{\phi}(\cdot | x_{\le t-1})}[\exp (-E_\theta(x_{\le t-1}, x_{t}',\cdots, x_T'))]}.
\end{align*}

\newpage
\begin{figure}[t]
    \centering
    \includegraphics[width=0.99\linewidth]{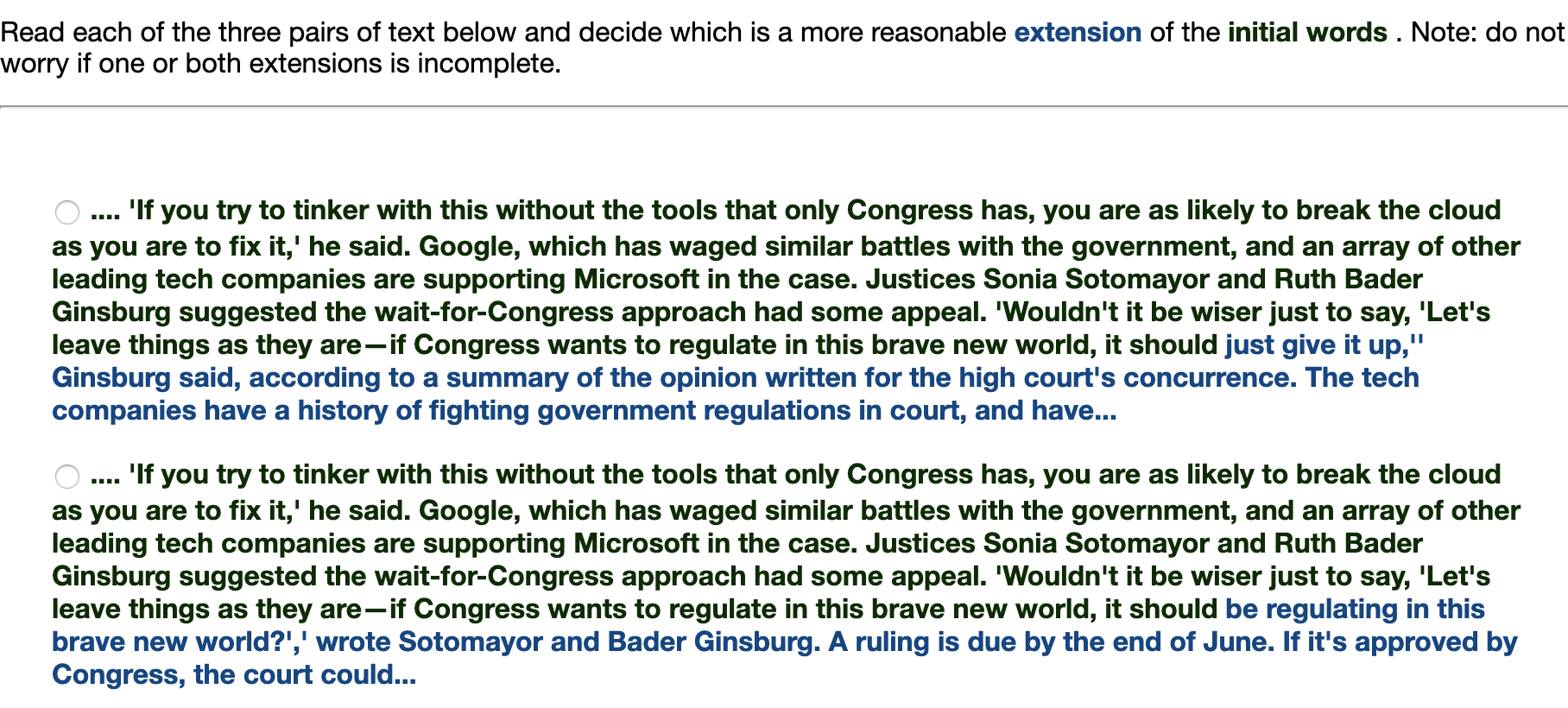}
    \caption{Screenshot of the human evaluation.}
    \label{fig:human}
\end{figure}

\section{Human Evaluation} \label{app:human}
A screenshot of the human evaluation experiments can be found in Fig~\ref{fig:human}.
Every page asks for 4 comparisons, one of which we know what the ground truth answer is.
We subsampled 333 sentences from the test set of CC-News, and asked 3 Amazon Mechanical turkers to vote.
We consider one continuation better if it gets more votes. To check the quality of the received ratings,
we performed a qualification task beforehand, where one of the continuations is real text,
and we kept the top half performing turkers for further evaluation (corresponding to higher than 66.67\% accuracy
for discriminating real from LM samples -- for a total of 26 qualified turkers).
Then in the actual experiment, we use one out of every four comparisons as an attention check and drop
responses if the turker did not pass the check.

\newpage
\section{\label{appendix:examples}Qualitative analysis}
We present generation examples when our approach \textsc{Base LM} outperforms baseline \textsc{BALM} in Table~\ref{tab:examples1}, and when our approach underperforms in Table~\ref{tab:examples2}. Here the judgment is based on human evaluation when all three turkers unanimously voted in favor of one model over the other.

\begin{table}[!htp]
  \tiny
    \centering
    \begin{tabular}{m{5.5cm}m{1.3cm}m{5.9cm}}
    \toprule
    Prefix & Model &  Generation   \\
    \midrule
    \multirow{5}{5.5cm}{
    \footnote{Excerpt from \url{https://chicago.cbslocal.com/2017/11/20/chance-the-rapper-jewel-osco/}}the timing of the donation couldn't be better, and even went on to say that no one is more relevant than Chance right now. Thank you to Doug and @jewelosco for donating \$1,000,000 to The New Chance Arts \& Literature Fund, which will ensure more students have access to arts enrichment education! \#supportcps pic.twitter.com/MXZtpscU5b — SocialWorks (@SocialWorks\_Chi) November 20, 2017 ``In the past, CPS hasn't really gotten a lot of money from non-profits} &\textsc{Base LM} &. And the fact that the money is coming from the government, it makes a big difference,'' he said. ``We're not the biggest donor of arts education, so to\\
     \cline{2-3}
     &\textsc{BALM} & , and it doesn't even have the kind of resources or funding for the arts or literary programs that we need to do the things that people want to support,'' said Chance.\\
     \cline{2-3}
     &\textsc{BiT-Base-Worst} &  ,” said Cpl. Chris Bissonnette of the Stony Brook Police Department. “So the reason is, as I think about the impact of these grants and these grants to schools \\
     \cline{2-3}
     &\textsc{BiT-Base} &  to fund programming. So, I think that it's a great time for us to do something to support that,'' said Chance. ``We want to make sure that our programs \\
     \cline{2-3}
     &\textsc{Gold} &  and charitable arms and organizations because there's been a past of corrupt activity in the city of Chicago,'' Chance said. ``And so, when it comes down to it, it \\
     \midrule
     \multirow{5}{5.5cm}{\footnote{Excerpt from  \url{https://chiefswire.usatoday.com/2017/12/10/halftime-analysis-for-chiefs-vs-raiders-week-14/}}quarter. The penalties are still somewhat of a problem but tight ends Travis Kelce and Demetrius Harris made some impressive grown-man football plays. -It was nice to see running back Kareem Hunt get in the end zone for the first time since Week 3. He must feel good to end the drought. -Kelce was visibly frustrated on the sidelines and rightfully so. The officials seem to be leaning toward Oakland with calls today. Still, Kelce should've caught that easy pass that he dropped. -Quarterback Alex Smith has delivered a couple of nice deep balls to } &\textsc{Base LM} & get his hands on the ball this week. He threw two touchdown passes on Thursday. He should get another touchdown as the season goes on. He's got a good chance to be one of \\
     \cline{2-3}
     &\textsc{BALM} &  the Chiefs and the Chiefs defense has made them pay for it. The offense has done a nice job with the ball, and they's been a factor. It's nice to have\\
     \cline{2-3}
     &\textsc{BiT-Base-Worst} &  tight end Martavis Bryant to get the ball back, but the throw from Smith to Davis Webb is incomplete, which could lead to an incompletion. -The Chiefs are now 5-2 after winning \\
     \cline{2-3}
     &\textsc{BiT-Base} &  help the Chiefs win this game. His completion of a deep ball to wide receiver Sammy Watkins is a nice sign. -Jalen Ramsey should be back in the lineup today. The Jaguars are going \\
     \cline{2-3}
     &\textsc{Gold} &  receivers Albert Wilson and Tyreek Hill. The footwork by Wilson on the deep ball was amazing and Hill's ball-tracking skills were on point as usual. Good things happen when Smith trusts \\
     \midrule
    \multirow{5}{5.5cm}{\footnote{Excerpt from \url{https://www.glamour.com/story/kourtney-kardashian-intense-detox}}has an amazing ability to detoxify itself namely by the liver, kidneys, and skin,'' she says. Detox programs and diets do sometimes make people feel better, but that happens more often if you're eating and/or drinking things that make feel like crap (something it doesn't sound like Kardashian was doing in the first place). And if, somehow, you discover that you have elevated levels of mercury or lead in your body like Kardashian said she had, it's important to figure out why, Dr. Wider says. ``Finding the source is extremely important for any heavy metal, so you} &\textsc{Base LM} & can't say 'no or I've never seen something of that value, so I'll try and find it again and see what happens.'' So don't be afraid to find a way to detox or \\
     \cline{2-3}
     &\textsc{BALM} &  want to get the full picture and not just have to look at the whole picture,'' she says. ``It can be a great idea to talk to your doctor to determine what's causing your symptoms."\\
     \cline{2-3}
     &\textsc{BiT-Base-Worst} &  can get rid of toxins that are found in foods like whole wheat bread, nuts like walnuts, walnuts and walnuts," she says. "If a source of the metals has elevated levels of \\
     \cline{2-3}
     &\textsc{BiT-Base} &  can't say, 'Oh my God, it's a lead, but I'm fine.' But it's definitely a problem in our bodies,'' she says. ``And if it's a problem in \\
     \cline{2-3}
     &\textsc{Gold} &  can lower your exposure over time,'' she says. Meaning, a detox diet probably isn't going to help you as much as, say, reducing your intake of foods that are high in mercury (like\\
         \bottomrule\\
    \end{tabular}
    \caption{Example generations when \textsc{BIT-BASE} outperforms \textsc{BALM} according to human evaluation. \textsc{BiT-Base-Worst} shows the LM sample with the highest energy score.}
    \label{tab:examples1}
\end{table}

\begin{table}[!htp]
  \tiny
    \centering
    \begin{tabular}{m{5.5cm}m{1.3cm}m{6.0cm}}
    \toprule
    Prefix & Model &  Generation   \\
    \midrule
    \multirow{5}{5.5cm}{\footnote{Excerpt from \url{https://www.libraryjournal.com/?detailStory=lakewood-oh-mom-sues-library-over-teens-rough-treatment}}but as the audio only came on halfway through the incident, it did not record his initial words to the girl or her response. It was also dropped in the hallway during their interaction, so it did not record any of what happened once the two left the library. Jones brought her outside to the library's front steps, where he waited for responding officers to arrive. They did not call for an ambulance, according to the lawsuit. Rather, when the officers arrived, they told her to spit blood on the grass instead of the concrete, and to move from the steps to a bench. Some 20 minutes} &\textsc{Base LM} &later, she was dead, according to the lawsuit. ``It's not uncommon for an officer to spit blood on an unarmed, vulnerable adult. The use of spit is illegal in a public place,''\\
     \cline{2-3}
     &\textsc{BALM} &later, Jones was arrested for disorderly conduct and released after posting bond, the suit said. She was arrested again on May 2. In the suit, Jones is seeking more than \$1 million. A\\
     \cline{2-3}
     &\textsc{BiT-Base-Worst} &  later, he walked up on the ground, where they began to yell and yell again. When Jones refused, the officers ordered him not to leave or leave any more. Jones then pulled his gun on\\
     \cline{2-3}
     &\textsc{BiT-Base} &  after the officers' arrival, Jones and the girl left the library and headed to the parking lot of a nearby apartment complex. ``This is the type of behavior that is most likely to cause harm to\\
     \cline{2-3}
     &\textsc{Gold} &  later, J.G's mother arrived and took her daughter to Cleveland Clinic in Lakewood, where she was treated for a dislocated jaw, a head injury, and an injury to her knee. \\
     \midrule
     \multirow{5}{5.5cm}{\footnote{Excerpt from \url{https://www.sun-sentinel.com/community/delray-sun/fl-drf-village-academy-steam-0418-20180410-story.html}}, Bronson said. ``The initiative provides a variety of supports to early childhood programs' children, families and staff. The resources provided through this partnership increase the quality of the participating programs, which benefits the community and impacts our future in such a positive way,'' Scott said. Visit PNCGrowUpGreat.com/donorschoose. $\backslash$nHere are Varsity sports headlines for April 13, 2018. Refresh to get the latest as we add news throughout the night as we collect scores: Best of the best in track and field Our Sentinel coverage area top performers lists for girls track and field  } &\textsc{Base LM} &  at the Varsity Track \& Field Invite.$\backslash$nThe U.S. Army Corps of Engineers has approved \$2 billion in construction work along the U.S.-Mexico boundary as a way to \\
     \cline{2-3}
     &\textsc{BALM} &  . Check back frequently for updates. The Sentinel also has a feature on the boys basketball and wrestling teams. Boys golf The Sentinel boys golf and wrestling teams have been one of those teams who has had some\\
     \cline{2-3}
     &\textsc{BiT-Base-Worst} &  .$\backslash$nLONDON, April 13 (IFR) - The following are some of the main factors expected to affect Spanish corporate bond prices on Wednesday. BAML: Spanish sovereign wealth fund PPV \\
     \cline{2-3}
     &\textsc{BiT-Base} &  .$\backslash$nA few weeks back, it's been rumored that the HTC Desire was going to be the company’s last flagship phone, and now, a new leak has confirmed that it \\
     \cline{2-3}
     &\textsc{Gold} & and boys track and field are updated going into the Saturday district meets. The season is heating up with more district and region races coming up next week. Click these links for girls top performers and boys top\\
     \midrule
    \vspace{-12pt}\multirow{5}{5.5cm}{
    \footnote{Excerpt from \url{https://seekingalpha.com/article/4215142-apple-looks-to-recharge-tech-sector}}leaders meeting November 17-18 in Papua New Guinea as potential Xi-Trump meet dates. If all else fails, Trump and Xi are also expected to meet for a bit at the G20 meeting at the end of November. On the economic calendar next week, the update on jobs and the U.S. trade deficit are the headliners on November 2. Notable earnings reports: Akamai Technologies (NASDAQ:AKAM), Mondelez International (NASDAQ:MDLZ) and Olin Corp. (NYSE:OLN) on October 29; Under Armour (NYSE:} &\textsc{Base LM} & UAA), General Motors (NYSE:GM) on November 4; and Procter \& Gamble (NYSE:PG) for October. On the retail front, Lowe's Companies (NYSE:L \\
     \cline{2-3}
     &\textsc{BALM} &  UA) on October 30; CVS Health (NASDAQ:CVS) on November 27; Intel Corporation (NASDAQ:INTC) on October 28; and Verizon Communications (NYSE:V\\
     \cline{2-3}
     &\textsc{BiT-Base-Worst} &  UAA) and Adidas (OTCPK:ADDYYF; OTCQX:ADDYYFGF; OLYMP), on November 30; and Qualcomm Incorporated (NASDAQ: \\
     \cline{2-3}
     &\textsc{BiT-Base} &   UAA), Johnson Controls (NYSE:JCI) and Cisco Systems (NASDAQ:CSCO) on November 6.$\backslash$nA woman who had to have her nose and mouth taped as punishment \\
     \cline{2-3}
     &\textsc{Gold} &  UAA), eBay (NASDAQ:EBAY), General Electric (NYSE:GE), Coca-Cola (NYSE:KO), Pfizer (NYSE:PFE) and Electronic Arts (NAS\\
         \bottomrule\\
    \end{tabular}
    \caption{Example generations when \textsc{BIT-BASE} underperforms \textsc{BALM} according to human evaluation. \textsc{BiT-Base-Worst} shows the LM sample with the highest energy score.}
    \label{tab:examples2}
\end{table}

\newpage
\bibliography{refs}

\end{document}